\documentclass[10pt,twocolumn]{IEEEtran}
\usepackage{graphicx}
\usepackage{epsfig}
\usepackage{amsmath}
\usepackage{bbm}
\usepackage{amssymb}
\usepackage{setspace}
\usepackage{wrapfig}
\usepackage{times,color}
\usepackage{caption}
\usepackage{subcaption}
\usepackage{algorithm}
\usepackage{algpseudocode}
\usepackage{epstopdf}
\usepackage{cite,footnote,xspace,syntonly,bm}

\usepackage{amsfonts}

\input{mysymbol.sty}

\algnewcommand{\Inputs}[1]{%
  \State \textbf{Inputs:}
  \Statex \hspace*{\algorithmicindent}\parbox[t]{.8\linewidth}{\raggedright #1}
}
\algnewcommand{\Initialize}[1]{%
  \State \textbf{Initialize:}
  \Statex \hspace*{\algorithmicindent}\parbox[t]{.8\linewidth}{\raggedright #1}
}

\newcommand{\calN}{{\cal N}}

\def \bv {{\bf v}}

\def \bv {{\bf v}}

\def \bbeta {{\boldsymbol \beta}}

\def \bbzeta {{\boldsymbol \zeta}}

 \long\def\symbolfootnote[#1]#2{\begingroup
 	\def\thefootnote{\fnsymbol{footnote}}
 	\footnote[#1]{#2}\endgroup} \psfull

\begin{document}

\title{\huge Kernel-Based Structural Equation Models for\\
Topology Identification of Directed Networks}

\author{{{\it Yanning Shen, \textit{Student Member}, \textit{IEEE}, Brian Baingana, \textit{Member}, \textit{IEEE},\\
 and Georgios~B.~Giannakis, \textit{Fellow, IEEE}}}\\
\IEEEauthorblockA{Dept. of ECE and DTC, University of Minnesota, Minneapolis, USA\\
Emails: \{shenx513,baing011,georgios\}@umn.edu}
}

\markboth{IEEE TRANSACTIONS ON SIGNAL PROCESSING \today \; (SUBMITTED)}{}
\maketitle  

\symbolfootnote[0]{$\dag$ Work in this paper was supported by grants NSF 1500713 and NIH 1R01GM104975-01. Part of this work has also been accepted for presentation at the Conference on Information Sciences and Systems, Princeton, New Jersey, March 2016.}

\symbolfootnote[0]{$\ast$ Y. Shen, B. Baingana, and G. B. Giannakis are with the Dept.
	of ECE and the Digital Technology Center, University of
	Minnesota, 200 Union Street SE, Minneapolis, MN 55455. Tel/fax:(612)626-7781/625-4583; Emails:
	\texttt{\{shenx513,baing011,georgios\}@umn.edu. }}



\vspace{-8mm}
\begin{abstract}
Structural equation models (SEMs) have been widely adopted for inference of causal interactions in complex networks. Recent examples include unveiling topologies of hidden causal networks over which processes such as spreading diseases, or rumors propagate. The appeal of SEMs in these settings stems from their simplicity and tractability, since they typically assume linear dependencies among observable variables. Acknowledging the limitations inherent to adopting linear models, the present paper advocates nonlinear SEMs, which account for (possible) nonlinear dependencies among network nodes. The advocated approach leverages kernels as a powerful encompassing framework for nonlinear modeling, and an efficient estimator with affordable tradeoffs is put forth. Interestingly, pursuit of the novel kernel-based approach yields a convex regularized estimator that promotes edge sparsity, and is amenable to proximal-splitting optimization methods. To this end, solvers with complementary merits are developed by leveraging the alternating direction method of multipliers, and proximal gradient iterations. Experiments conducted on simulated data demonstrate that the novel approach outperforms linear SEMs with respect to edge detection errors. Furthermore, tests on a real gene expression dataset unveil interesting new edges that were not revealed by linear SEMs, which could shed more light on regulatory behavior of human genes.
\end{abstract}


\vspace{3mm}
\section{Introduction}
\label{sec:intro}
Inference of network topologies is a well-studied problem with practical applications in diverse settings; see e.g., \cite[Ch. 7]{kolaczyk2009statistical} and references therein. For example, discovery of causal links between regions of interest in the brain is tantamount to identification of an implicit connectivity network. Studies pertaining to regulatory and inhibitory interactions among genes depend upon identification of unknown links within gene regulatory networks. Suspected terrorists can be identified by learning hidden links in social interaction networks, or telephone call graphs connecting individuals to well-known terrorists. 

SEMs provide an overarching statistical modeling framework for inference of causal relationships within complex systems~\cite{kaplan09}. These directional effects are seldom revealed by standard linear models which impose symmetric associations between random variables, such as those represented by covariances or correlations; see e.g.,~\cite{friedman2008sparse}. Most contemporary approaches overwhelmingly focus on linear SEMs due to their inherent simplicity and tractability. For this reason, linear SEMs have been widely adopted in fields as diverse as sociometrics~\cite{goldberger1972structural}, psychometrics~\cite{muthen1984general}, and genetics~\cite{cai2013inference}. More recently, linear SEMs have been advocated for tracking dynamic topologies of  directed social networks from temporal cascade traces observed over the nodes~\cite{BGG14,baingana2015switched}. 

Recognizing the limitations of linear SEMs for modeling nonlinear phenomena, several variants of nonlinear SEMs have recently emerged in several works; see e.g.,~\cite{joreskog1996nonlinear,wall2000estimation,jiang2010bayesian,lee2003model,harring2012comparison,kelava2014nonlinear}. For example, in~\cite{lee2003model} and~\cite{lee2003maximum}, nonlinearities are only modeled with respect to independent system variables, while a hierarchical Bayesian framework is reported in~\cite{jiang2010bayesian}. Polynomial SEMs, the closest extension to classical linear SEMs, are the focus of several other works; see e.g.,~\cite{joreskog1996nonlinear,wall2000estimation,harring2012comparison,kelava2014nonlinear}. In all these contemporary approaches, it is assumed that the network connectivity structure is known a priori, and developed algorithms only estimate the unknown edge weights. However, this is a rather significant limitation since such prior information may not be available in practice, especially when dealing with massive networks. 

The present paper builds upon these prior works, and advocates an additive nonlinear model to capture dependencies between observed nodal measurements, without explicit knowledge of the edge structure. A key feature of the novel approach is the premise that edges in the hidden network are sparse. Edge sparsity is an attribute that naturally emerges in most real-world networks; for instance, genes are known to regulate or inhibit only a small subset of all possible genes within an organism. Members of online social networks (e.g., \emph{Facebook}, or \emph{Twitter}) typically connect with a few hundreds of friends, compared to millions of all other subscribers. This sparse edge connectivity has recently motivated development of regularized estimators, promoting the inference of sparse network adjacency matrices~\cite{BGG14,angelosante2011sparse,kopsinis2011online}. Indeed, it has been shown in these works that exploiting edge sparsity markedly improves the estimation accuracy of topology inference algorithms. 

The rest of the paper is organized as follows. Section~\ref{sec:nmps} postulates an additive nonlinear model, with functional forms of the nonlinear summands considered unknown. Toward promoting edge sparsity, Section~\ref{sec:kernels} adopts a sparsity-promoting regularized least-squares estimator, where it turns out that nonlinear data dependencies are all captured through inner products, eliminating the need to explicitly know their functional forms. To this end, a regularized kernel-based approach is advocated, leading to a convex optimization problem that can be solved with global optimality guarantees. 

Exploiting the problem structure, solvers that leverage proximal-splitting optimization approaches with complementary merits are developed in Section~\ref{sec:algs}. Specifically, an alternating direction method of multipliers (ADMM) solver that is amenable to decentralized implementation, and a first-order proximal gradient (PG) algorithm with reduced computational complexity are put forth. Building on prior work in~\cite{shen2016nonlinear}, Section~\ref{sec:polysem} specializes the postulated nonlinear SEM to the polynomial form, with nonlinear summands postulated to belong to the class of higher-order polynomials. Extensive numerical tests on both simulated and real gene expression data are conducted in Section~\ref{sec:test}, and the novel approach is shown to outperform linear variants with respect to the number of correctly identified edges. Finally, Section~\ref{sec:conclusion} concludes the paper and highlights several potential future research directions that this work opens up.

\noindent\textit{Notation}. Bold uppercase (lowercase) letters will denote matrices (column vectors), while operators $(\cdot)^{\top}$, $\lambda_{\max}(\cdot)$, and $\textrm{diag}(\cdot)$ will stand for matrix transposition, maximum eigenvalue, and diagonal matrix, respectively. The identity matrix will be represented by $\mathbf{I}$, while $\mathbf{0}$ will denote the matrix of all zeros, and their dimensions will be clear in context. Finally, the $\ell_p$ and Frobenius norms will be denoted by $\|\cdot\|_p$, and $\|\cdot\|_F$, respectively. 

\vspace{-0.2cm}
\section{Preliminaries}
\label{sec:pre}
\begin{figure*}[tpb!]
	\centering
	\includegraphics[width=14cm]{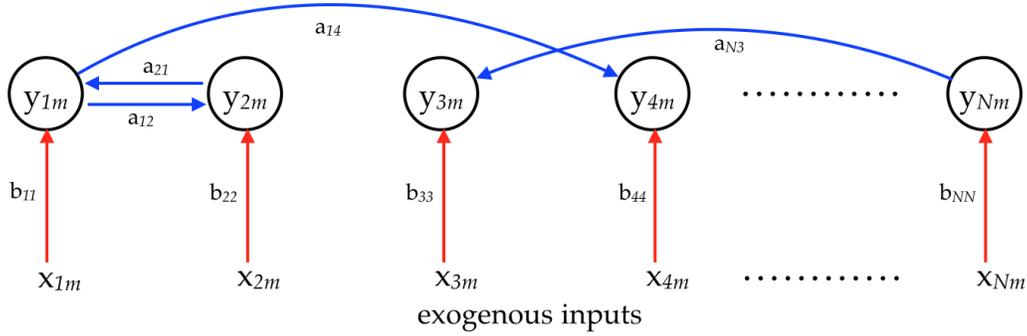}
	\caption{Illustration of an $N$-node network with directed edges (in blue), and the $m$-th sample of endogenous measurements per node. SEMs explicitly account for exogenous inputs (red arrows), upon which endogenous variables may depend, in addition to the underlying network topology.}\label{fig:graph}
\end{figure*}
\vspace{-1mm}
Consider a directed $N$-node network $\mathcal{G}(\mathcal{V},\mathcal{E})$, whose topology is captured by an unknown adjacency matrix $\bbA\in\mathbb{R}^{N\times N}$. Let $a_{ij}\in\{0, 1\}$ denote entry $(i,j)$ of $\bbA$, which is nonzero only if a directed edge is present from node $i$ to $j$; see Figure~\ref{fig:graph}. Suppose the network represents an abstraction of a complex system with measurable inputs and an observable output process that propagates over the network following directed links. Let $x_{im}$ denote the $m$-th input to node $i$, and $y_{im}$ the $m$-th observation of the propagating process measured at node $i$. In the context of brain networks, $y_{im}$ could represent the $m$-th sample of an electroencephalogram (EEG), or functional magnetic resonance imaging (fMRI) measurement at region $i$, while $x_{im}$ could be a controlled stimulus that affects a specific region of the brain. In social networks (e.g., Twitter) over which information diffuses, $y_{im}$ could represent the timestamp when subscriber $i$ tweeted about a viral story $m$, while $x_{im}$ could measure their level of interest in such stories.

SEMs postulate that $y_{im}$ depends on two classes of variables: i) the measurements of the diffusing process $\{ y_{jm} \}_{j \neq i}$  (a.k.a. \emph{endogenous} variables); and ii) external inputs $x_{im} $ (a.k.a. \emph{exogenous} variables). Most contemporary SEM approaches assume that $y_{im}$ depends linearly on both $\{ y_{jm} \}_{j \neq i}$ and $ x_{im} $, that is,
 \begin{align}
 	y_{jm}= \underbrace{\sum_{i\neq j}a_{ij}y_{im}}_{\text{endogenous term}} + \underbrace{b_{jj} x_{jm} }_{\text{exogenous term}}+e_{jm}
	\label{eq:sem:linear1}
 \end{align}
where $e_{jm}$ captures unmodeled dynamics. Note that~\eqref{eq:sem:linear1} encodes the causal dependence structure between network nodes through the unknown coefficients $\{ a_{ij}\}$. Specifically, $y_{jm}$ only depends on the single-hop neighborhood of node $j$, namely $\mathcal{N}_j := \{ i : a_{ij} \neq 0 \}$. In the sequel, this paper will build upon the classical linear SEM~\eqref{eq:sem:linear1}, and put forth a general nonlinear SEM. Before proceeding, a couple of pertinent remarks are due.

\begin{remark}[Structural versus measurement models] 
Strictly speaking,~\eqref{eq:sem:linear1} represents the so-termed structural model in general SEM parlance, where the endogenous and exogenous variables may not be directly observable (a.k.a. structural/latent variables); see e.g., \cite[Ch. 2]{kaplan09} for details. In this general setup, the complete SEM specification also endows the model with noisy measurement equations, through which the structural variables are observed. Nevertheless, the structural variables are directly observable in a number of practical settings, and the model~\eqref{eq:sem:linear1} often suffices. The rest of this paper deals with this latter setting, and only focuses on the dependence model between endogenous and exogenous variables.
\end{remark}


\begin{remark}[Link with sparse subspace clustering (SSC)] 
In the absence of exogenous inputs,~\eqref{eq:sem:linear1} bears remarkable similarity to SSC~\cite{elhamifar2009sparse}, whose goal is to cluster high-dimensional data points lying within a union of low-dimensional subspaces. Given vectors $\{ \mathbf{y}_i \in \mathbb{R}^D \}_{i=1}^N$ sampled from the union of $d$-dimensional subspaces embedded in $\mathbb{R}^{D}$, with $d \ll D$, SSC postulates that each datum is expressible as a linear combination of points drawn from its subspace, that is, $\bby_i = \sum_j w_{ij} \bby_j + \boldsymbol{\epsilon}_i$, where $w_{ij} \neq 0$ only if $i$ and $j$ belong to the same subspace, while $\boldsymbol{\epsilon}_i$ captures noise and unmodeled dynamics. SSC prescribes a sparsity-promoting least-squares estimator for the weights $\{ w_{ij} \}$ by solving: $\min_{\{w_{ij},~ j\neq i\}} \| \mathbf{y}_i - \sum_{j } w_{ij} \mathbf{y}_j  \|_2^2+\lambda \sum_{j}|w_{ij}|$, which ensures that only a few coefficients are nonzero per $i$. This is followed by spectral clustering of the similarity graph induced by the estimated pairwise weights, with subsequent identification of the constituent subspaces e.g., via principal component analysis. Clearly, estimation of the weights in SSC is reminiscent of recovery of the unknown coefficients $\{ a_{ij}\}$ in~\eqref{eq:sem:linear1}. Viewing SSC as an approximate linear approach to manifold learning [compare also with local linear embedding in e.g.,~\cite{roweis2000nonlinear}], the advocated nonlinear SEM in the present paper can be readily generalized to tasks where data are actually sampled from low-dimensional nonlinear manifolds, embedded in high-dimensional Euclidean spaces.
\end{remark}


\section{Nonlinear Model and Problem Statement}
\label{sec:nmps}

Going beyond the linearity assumption, this section generalizes~\eqref{eq:sem:linear1} to capture nonlinear causal dependencies. Consistent with traditional linear SEMs, it is postulated that each nodal measurement results from a combination of endogenous and exogenous terms. Specifically, let
 \begin{align}
 	y_{jm} = \psi_{\text{en}} (\bby_{-jm}) + \psi_{\text{ex}} (x_{jm})  + e_{jm}
	\label{eq:sem:nonlinear0}
 \end{align}
where $\psi_{\text{en}}(.) $ and $\psi_{\text{ex}}(.) $ are memoryless nonlinear functions of their arguments, and $\bby_{-jm} := [ y_{1m}, \dots, y_{(j-1)m},$ $y_{(j+1)m}, \dots, y_{Nm}  ]^{\top}$ collects all endogenous variables per sample $m$, excluding the measurements at node $i$. For simplicity of exposition, it will be explicitly assumed here that $\psi_{\text{ex}} (x_{jm}) = b_{jj} x_{jm}$, although the advocated approach readily incorporates settings where $\psi_{\text{ex}} (.)$ admits a more general nonlinear form. Similar to the so-termed \emph{generalized additive models} (GAMs), the present paper postulates that
 \begin{align}
 	\psi_{\text{en}} (\bby_{-jm})  = \sum_{i \neq j} a_{ij} \psi_{\text{en}, ij} (y_{im})
	\label{eq:sem:nonlinear1}
 \end{align}
where each $ \psi_{\text{en}, ij} (y_{im})$ is a nonlinear function of a single scalar variable. GAMs are popular for nonlinear modeling in diverse machine learning tasks including regression and classification, since they lead to computationally efficient solvers, while their additive nature facilitates assessment of feature importance in prediction tasks. Since the hidden network topology is encoded through the unknown coefficients in~\eqref{eq:sem:nonlinear0}, it is desirable to adopt a nonlinear modeling framework in which the unknowns admit a similar interpretation. The key difference between~\eqref{eq:sem:nonlinear1} and traditional GAMs is the introduction of the explicit unknown coefficients $\{ a_{ij} \}$ to facilitate identification of the network structure.

Furthermore, suppose that $ \psi_{\text{en}, ij} (y_{im})$ can be written as
 \begin{align}
  \psi_{\text{en}, ij} (y_{im}) = \sum_{p = 1}^P   c_{ijp} \phi_{p}(y_{im}) 
	\label{eq:sem:nonlinear2}
 \end{align}
where $\{c_{ijp}\}_{p=1}^P$ are unknown coefficients but $\{ \phi_{p}(.)  \}_{p=1}^P$ are (possibly) known functions, and $P$ can take on any positive integer or even infinity. Collecting the $P$ coefficients into $\bbc_{ij}:=[c_{ij1},  \ldots, c_{ijP}]^\top$, and defining $\bbphi (y_{im}):=[\phi_{1}(y_{im}), \ldots, \phi_{P}(y_{im})]^{\top}$, then one obtains the following nonlinear SEM [cf. \eqref{eq:sem:nonlinear1} and \eqref{eq:sem:nonlinear2}]
 \begin{align}
 	y_{jm}=\sum_{i\neq j}a_{ij}  \bbphi^\top(y_{im})\bbc_{ij}+b_{jj}x_{jm}+e_{jm}
	\label{eq:sem:nonlinear}
 \end{align}
for $j = 1, \dots, N$, and $m = 1, \dots, M$. With $\bby_j := [y_{j1}, \ldots, y_{jM}]^\top \in\mathbb{R}^M$, $\bbx_j := [x_{j1},\ldots,x_{jM}]^\top$, and $\bbe_j:= [e_{j1},\ldots,e_{jM}]^\top$,~\eqref{eq:sem:nonlinear} can be written in vector form as

 \begin{align}
 	{\bby}_j=\sum_{i\neq j}\bbPhi_{i}\bbc_{ij}a_{ij}+b_{jj}\bbx_j+\bbe_j
 \end{align} 
where $\bbPhi_{i}  := [\bbphi(y_{i1}),\ldots,\bbphi(y_{iM})]^\top$. Furthermore, letting $\bbY:=[\bby_1, \dots,\bby_N]$, $\bbX:=[\bbx_1,\dots,\bbx_N]$ and $\bbPhi:=[\bbPhi_1, \dots,\bbPhi_N]$, one obtains the following nonlinear matrix SEM
\begin{align}
	\bbY=\bbPhi\bbW+\bbX\bbB+\bbE
	\label{mod:nonlinear}
\end{align}
where $\mathbf{B}$ is a diagonal matrix with $(i, i)$th entry $\left[ \mathbf{B} \right]_{ii} = b_{ii}$, and the $NP \times N$ block matrix $\bbW$ is 
\begin{align}
\bbW:=\left[
\begin{array}{cll}
a_{11}\bbc_{11} &	\cdots &a_{1N}\bbc_{1N}\\
\vdots & \ddots & \vdots\\
a_{N1}\bbc_{N1} &\cdots	& a_{NN}\bbc_{NN}\end{array} \right]\label{def:W}
\end{align}
has structure modulated by the entries of $\mathbf{A}$. For instance, if $a_{ij}=0$, then $a_{ij} \mathbf{c}_{ij}$ is an all-zero block regardless of the values of entries in $\mathbf{c}_{ij}$. 

As expected in real-world networks that exhibit edge sparsity, $\bbA$ has only a few nonzero entries, and this prior information will be useful to develop efficient estimators. The dependence of \eqref{mod:nonlinear} on $\bbA$ is indirect, through $\bbW$. Interestingly, the sparsity inherent to $\bbA$ is not lost in this transformation, and it manifests itself as \emph{group sparsity} in the block-matrix $\bbW$. Specifically, entries $\{ a_{ij} \}$ determine whether certain blocks are all-zero or not [cf. \eqref{def:W}], naturally leading $\bbW$ to exhibit group sparsity. The present paper will leverage this inherent group sparsity, and put forth a novel kernel-based estimator for inference of the unknown network topology, in the next section. Having introduced the proposed nonlinear SEM, the problem statement can now be formally stated as follows.

\noindent\textbf{Problem statement.} Given nodal measurements $\bbX$ and $\bbY$, we wish to estimate the edge-modulated matrix $\bbW$, and correspondingly the unknown adjacency matrix $\mathbf{A}$, as well as the diagonal exogenous coefficient matrix $\bbB$. Whether $\bbPhi$ is (un)known will be clarified in the ensuing section which puts forth a novel estimator.
%
%

\section{Kernel-based topology estimation}
\label{sec:kernels}
\vspace{-3mm}
Towards estimating the unknown variables in~\eqref{mod:nonlinear}, with no knowledge of the additive noise statistics, this section advocates minimization of a regularized least-squares (LS) cost, namely
\vspace{0.2cm}
\begin{equation}
\label{eq:obj:matrix}
	\{\hat{\bbW},\hat{\bbB}\} = \underset{\bbW \in \mathcal{W}, \; \bbB \in \mathcal{B}}{\arg \min}   (1/2) \|\bbY-\bbPhi\bbW-\bbX\bbB\|_F^2 + \lambda \|\bbW\|_{\mathcal{G},1}
\end{equation} 
where $\mathcal{B} := \{ \bbB \in \mathbb{R}^{N \times N}: b_{ij} = 0, \; \forall i \neq j  \}$ is a set of diagonal square matrices, and $\mathcal{W} := \{ \bbW \in \mathbb{R}^{NP \times N} :  \bbw_{ii} = \mathbf{0}, \;\; i = 1, \dots, N \}$, with $\bbw_{ij}:= a_{ij}\bbc_{ij}$ denoting the $(i,j)-$th block entry of $\bbW$. The constraint set $\mathcal{W}$ restricts solutions to adjacency matrices representing networks without self-loops, i.e., $a_{ii} = 0 \Longleftrightarrow \bbw_{ii} = \mathbf{0}$.
On the other hand, restricting $\bbB$ to the set of diagonal matrices $\mathcal{B}$ is consistent with the prior modeling assumption that each node in the network is associated with a single exogenous input. The penalty term
\vspace{0.2cm}
\begin{align}
\label{eq:gpenalty}
	 \|\bbW\|_{\mathcal{G},1}:=\sum_{i,j}\|\bbw_{ij}\|_2
\end{align}
is a well-known regularizer that has been shown to promote group sparsity \cite{yuan2006model}, while the regularization parameter $\lambda \geq 0$ allows the estimator to trade off group sparsity for the LS fit. 

Upon solving the optimization problem~\eqref{eq:obj:matrix}, the resulting estimate $\hat{\bbW}$ contains all the information necessary to recover the network topology captured by $\mathbf{A}$, by simply identifying nonzero blocks. Although the penalty~\eqref{eq:gpenalty} is non-differentiable, it is worth stressing that the regularized cost in~\eqref{eq:obj:matrix} is jointly convex with respect to (w.r.t.) $\bbW$ and $\bbB$. In principle, the problem can be solved with guaranteed global optimality. However, this presupposes that $\bbPhi$ is readily available, which may not be possible since one may not know the functions $\{ \phi_p(.) \}_{p=1}^P$ explicitly. Moreover, even when $\{ \phi_p(.) \}_{p=1}^P$ are known, no efficient strategy may be available to solve~\eqref{eq:obj:matrix}, when $P\rightarrow \infty$. In lieu of the aforementioned challenges, this paper leverages kernel-based approaches that have well-appreciated merits in nonlinear modeling, often circumventing the need for explicit knowledge of the nonlinear mappings. Derivation of the novel kernel-based estimator will rely on the following result, which can be viewed as a variant of the known Representer's Theorem\cite{wahba1990spline} in our nonlinear SEM setting.

\vspace{1mm}
\begin{proposition}\label{proposition1}
Suppose $\hat{\bbW}$ is the optimal solution of the regularized LS estimator in~\eqref{eq:obj:matrix}, with $\hat{\bbw}_{ij}$ denoting its $(i,j)$-th block entry, then $\hat{\bbw}_{ij}$ can be written as
\vspace{2mm}
\begin{align}
\label{eq:sol:w}
	\hat{\bbw}_{ij}=\sum_{m=1}^M \alpha_{ijm} \bbphi(y_{im})=\bbPhi_i^\top \bbalpha_{ij}
\end{align}
where $\bbalpha_{ij}:=[\alpha_{ij1}, \ldots,\alpha_{ijM}]^\top \in \mathbb{R}^M$ is a coefficient vector.
\end{proposition}

\begin{proof}
See Appendix ~\ref{appendix:A} for the proof of Proposition~\ref{proposition1}.
\end{proof}

\noindent Acknowledging Proposition~\ref{proposition1}, and substituting~\eqref{eq:sol:w} into~\eqref{eq:obj:matrix}, the LS term can be written as
\begin{multline}
\label{eq:prob:ls_part}
(1/2) \|\bbY-\bbPhi\bbW-\bbX\bbB\|_F^2  \\
	= \sum_{j=1}^N(1/2) \bigg \| \bby_j-\sum_{i\neq j}\bbPhi_i\bbPhi_i^\top \bbalpha_{ij}-b_{jj}\bbx_j \bigg \|_2^2 
\end{multline}
while the regularization term is expressible as
\begin{align}
\label{eq:prob:reg_part}
\lambda \|\bbW\|_{\mathcal{G},1} = \lambda \sum_{j=1}^N\sum_{i\neq j}\sqrt{ \bbalpha_{ij}^\top\bbPhi_i\bbPhi_i^\top \bbalpha_{ij}}.
\end{align}
Clearly, each entry of $\bbPhi_i\bbPhi_i^\top$ constitutes an inner product in the ``lifted'' space, that is, $[\bbPhi_i\bbPhi_i^\top]_{k,l}=\bbphi^\top(y_{ik})\bbphi(y_{il})$. Defining the set of kernel matrices $\{\bbK_i\in \mathbb{R}^{M\times M}\}_{i=1}^N$, with $\bbK_i := \bbPhi_i\bbPhi_i^\top$, it is possible to recast the estimator~\eqref{eq:obj:matrix}, so that all dependencies on the functions $\{ \phi(.) \}$ are captured through entries of $\bbK_i $, for $i=1, \dots, N$. Accordingly, one obtains the following estimator
\begin{multline}
\label{eq:prob:mod}
\{ \{\hat{\bbalpha}_{ij} \}, \hat{\bbB} \} = \underset{\{\bbalpha_{ij}\}, \bbB }{\text{arg min}}  \;\; 
	(1/2) \|\bbY-\tilde{\bbK} \bbW_{\alpha}-\bbX\bbB\|_F^2 \\
+ \lambda \sum_{j=1}^N\sum_{i\neq j}\sqrt{ \bbalpha_{ij}^\top\bbK_i\bbalpha_{ij}}
\end{multline} 
where $\tilde{\bbK}:=[\bbK_1,\ldots,\bbK_N]$, $\bbW_{\alpha}:=[\bbalpha_{1},\ldots,\bbalpha_{N}]$, and $\bbalpha_{j}:=[\bbalpha_{1j}^\top,\ldots,\bbalpha_{(j-1) j}^\top,\boldsymbol{0}_{M\times 1}^{\top},\bbalpha_{j+1}^\top,\ldots\bbalpha_N^\top]^\top$. Examination of~\eqref{eq:prob:mod} reveals that $\bbW_{\alpha}$ inherits the block-sparse structure of $\bbW$, that is, if $\bbw_{ij}=\boldsymbol{0}$, then correspondingly $\bbalpha_{ij}=\boldsymbol{0}$. Recognizing that summands in the regularization term of~\eqref{eq:prob:mod} can be written as $\sqrt{ \bbalpha_{ij}^\top\bbK_i\bbalpha_{ij}} = \| \bbK_i^{1/2} \bbalpha_{ij} \|_2$, it is clear that \eqref{eq:prob:mod} is a convex problem admitting a globally-optimal solution. Exploiting the structure inherent to~\eqref{eq:prob:mod}, the next section develops inference algorithms for unveiling the unknown network topology.
\vspace{1mm}
\section{Topology Inference Algorithms}
\label{sec:algs}
Given matrices $\bbY$, $\bbX$, and $\tilde{\bbK}$, this section capitalizes on convexity, as well as the nature of the additive terms in~\eqref{eq:prob:mod} to develop efficient topology inference algorithms. Proximal-splitting approaches have been shown useful for convex optimization when the cost function comprises both smooth and nonsmooth components \cite{daubechies2004iterative}. Prominent among these are the alternating direction method of multipliers (ADMM), as well as proximal gradient (PG) descent approaches. Our first advocated approach leverages ADMM iterations as demonstrated next, see e.g., \cite{schizas3} for an early application of ADMM to distributed estimation.

\vspace{-1mm}
\subsection{Alternating Direction Method of Multipliers}
For ease of exposition, let the equality constraints ($\bbalpha_{jj}=\boldsymbol{0}$, $b_{ij}=0$, $\forall i\neq j$) temporarily remain implicit. Introducing the change of variables $\bbgamma_{ij} = \bbK_{ij}^{1/2}\bbalpha_{ij} $, problem~\eqref{eq:prob:mod} can be equivalently recast as
\begin{eqnarray}
\label{eq:prob:admm}
\nonumber
\{\{\hat{\bbalpha}_{ij}\}_{i\neq j}, \hat{\bbB}\}   =& \quad \\
\nonumber
\underset{\{\bbalpha_{ij}\}, \bbB }{\text{arg min}}  & 
	(1/2) \|\bbY-\tilde{\bbK} \bbW_{\alpha}-\bbX\bbB\|_F^2+g(\bbGamma) \\
 \text{s.t.}  &  \bbgamma_{ij}-\bbK_{ij}^{1/2}\bbalpha_{ij}=\boldsymbol{0} \;\;\forall i, \; j
\end{eqnarray}
where the matrix $ \bbGamma:=[\bbgamma_{1}, \ldots, \bbgamma_{N}]$, with $\bbgamma_{j}:=[\bbgamma_{1j}^\top, \ldots, \bbgamma_{(j-1) j}^\top,\boldsymbol{0}_{M\times 1}^{\top}, \bbgamma_{j+1}^\top, \ldots, \bbgamma_{N}^\top]^\top$, and
\begin{align}
 g(\bbGamma):=\lambda \sum_{i=1}^N\sum_{j=1}^N\|\bbgamma_{ij}\|_2.
 \end{align} 
Defining the block-diagonal matrix $\bbD{:=}\text{Bdiag}(\bbK_1, \dots, \bbK_N)$, where the operator $\text{Bdiag(.)}$ constructs a block-diagonal matrix from its matrix arguments, one can write the augmented Lagrangian of~\eqref{eq:prob:admm} as
\begin{align}
\label{eq:prob:aug_lagrangian}
	\mathcal{L}_{\rho}(\bbW_{\alpha}, \bbB, \bbGamma, \boldsymbol{\Xi}) 
	&=
	(1/2)\|\bbY-\tilde{\bbK} \bbW_{\alpha}-\bbX\bbB\|_F^2 + g(\bbGamma) \nonumber\\
&\hspace{-2cm}+ \langle \boldsymbol{\Xi} , \bbD^{1/2} \bbW_{\alpha}-\bbGamma \rangle + (\rho/2) \|\bbGamma - \bbD^{1/2}\bbW_{\alpha}\|_F^2.
\end{align}
In~\eqref{eq:prob:aug_lagrangian}, $\boldsymbol{\Xi}$ is a dual matrix-variable that collects all Lagrange multipliers corresponding to the equality constraints introduced in~\eqref{eq:prob:admm}, $\langle \mathbf{P}, \mathbf{Q} \rangle$ denotes the inner product between matrices $\mathbf{P}$ and $\mathbf{Q}$, while $\rho > 0$ is the a priori prescribed penalty parameter. ADMM essentially adopts alternating minimization (AM) iterations to minimize $\mathcal{L}_{\rho}(\bbW_{\alpha}, \bbB, \bbGamma, \boldsymbol{\Xi})$ over the primal variables $\bbW_{\alpha}, \bbB$, and $\bbGamma$, followed by a gradient ascent step over the dual variables in $\boldsymbol{\Xi}$\cite{baingana2013dynamic,schizas3}. During iteration $k+1$, AM updates of the primal and dual variables entail the following steps
\begin{subequations}
	\begin{align}
	\label{admm:1}
     \bbW_{\alpha}[k+1] & =  \underset{ \bbW_{\alpha} }{\arg \min}  \;\; \mathcal{L}_{\rho}(\bbW_{\alpha}, \bbB[k], \bbGamma[k], \boldsymbol{\Xi}[k])\\
	\label{admm:2}
	\bbB[k+1] & =  \underset{\bbB}{\arg\min} \;\; \mathcal{L}_{\rho}(\bbW_{\alpha}[k+1], \bbB, \bbGamma[k], \boldsymbol{\Xi}[k])\\
	\label{admm:3}
	\bbGamma[k+1] & =  \underset{\bbGamma}{\arg \min} \;\; \mathcal{L}_{\rho}(\bbW_{\alpha}[k+1], \bbB[k+1], \bbGamma, \boldsymbol{\Xi}[k])\\
	\label{admm:4}
	\boldsymbol{\Xi}[k+1] & =  \boldsymbol{\Xi}[k]+\rho (\bbD^{1/2}\bbW_{\alpha}[k+1] - \bbGamma[k+1]).
	\end{align}
\end{subequations}
Per step, the augmented Lagrangian is minimized w.r.t. a specific variable, with all the rest fixed to their most recent update, until convergence is attained. Each subproblem under the ADMM updates is studied next, and their solutions are correspondingly derived.  

Focusing on $\bbW_{\alpha}[k+1] $, note that \eqref{admm:1} can be written in terms of $\{ \bbalpha_j \}_{j=1}^N$ as 
\vspace{-0.2cm}
\begin{multline}
\label{admm:obj:F}
	\bbW_{\alpha}[k+1]  \; = \;\underset{\bbalpha_1, \ldots, \bbalpha_N }{\arg \min}  \;\;\;
	\sum\limits_{j=1}^N  \bigg[ 
	(1/2) \bbalpha_j^{\top} \left( \tilde{\bbK}^{\top} \tilde{\bbK} + \rho\bbD \right)
	\bbalpha_j \\ 
	\hspace{-0cm}-\bbalpha_j^{\top} \big( \rho \bbD^{1/2}\bbgamma_j[k]+\tilde{\bbK}^\top\bby_j 	 
-\bbD^{1/2}\bbxi_{j}[k]-b_{jj}[k]\tilde{\bbK}^\top\bbx_j \big)  \bigg]
\end{multline}
where $\boldsymbol{\xi}_j$ denotes column $j$ of $\bbXi$.
Clearly the cost in \eqref{admm:obj:F} decouples across columns of $\bbW_{\alpha}$, and admits closed-form, parallelizable solutions. Incorporating the structural equality constraint $\bbalpha_{jj}=\boldsymbol{0}$, one obtains the following decoupled subproblem per column $j$
\begin{align}
\label{admm:obj:alpha}
	\tilde{\bbalpha}_j[k+1] = \arg \min_{\tilde{\bbalpha_j}} \;\;\; (1/2) \tilde{\bbalpha}_j^\top\left(\tilde{\bbK}_j^\top\tilde{\bbK}_j+\rho\bbD_j\right)\tilde{\bbalpha}_j \nonumber\\
	-\tilde{\bbalpha}_j^\top\bbq_j[k]
\end{align}
where $\tilde{\bbalpha}_j$ denotes the $(N-1)M\times 1$ vector obtained by removing the entries of $\bbalpha_j$ indexed by $\mathcal{I}_j:=\{(j-1)M+1, \ldots, jM\}$. Similarly, $\tilde{\bbK}_j$ collects columns of $\tilde{\bbK}$ excluding the columns indexed by $\mathcal{I}_j$, the block-diagonal matrix $\bbD_j$ is obtained by eliminating rows and columns of $\bbD$ indexed by $\mathcal{I}_j$, while $\bbq_j[k]$ is constructed by removal of entries indexed by $\mathcal{I}_j$ from $\rho \bbD^{1/2}\bbgamma_j[k]+\tilde{\bbK}^\top\bby_j-\bbD^{1/2}\bbxi_{j}[k]-b_{jj}[k]\tilde{\bbK}^\top\bbx_j$. The per-column subproblem~\eqref{admm:obj:alpha} is an unconstrained quadratic optimization problem which, assuming $\left(\tilde{\bbK}_j^\top \tilde{\bbK}_j+\rho \bbD_j\right)$ is invertible, admits the following closed-form solution per $j$
\begin{align}
\label{admm:alpha1}
	\tilde{\bbalpha}_{j}[k+1]=&\left(\tilde{\bbK}_j^\top \tilde{\bbK}_j+\rho \bbD_j\right)^{-1}\bbq_j[k].
\end{align}
Upon evaluation of~\eqref{admm:alpha1}, column $j$ of $\bbW_{\alpha}[k+1]$ is updated by zero-padding $\tilde{\bbalpha}_{j}[k+1]$, that is,
\begin{align}
\label{admm:alpha}
	\bbalpha_j[k+1]= [\tilde{\bbalpha}_{1j}^\top , \ldots , \tilde{\bbalpha}_{(j-1)j}^\top, 
	\boldsymbol{0}_{M\times 1}^{\top}, 
	\tilde{\bbalpha}_{(j+1)j}^\top, \ldots, \tilde{\bbalpha}_{Nj}^\top]^{\top}
\end{align}
where $\boldsymbol{0}_{M \times 1}$ denotes the $M\times 1$ all-zero vector. 

The per-column update \eqref{admm:alpha} entails inversion of $M(N-1)\times M(N-1)$ matrices, which quickly becomes computationally prohibitive as the network grows in size. To circumvent this computational burden, we adopt the matrix inversion lemma, by recognizing that
\begin{align}
\label{eq:mat:inv}
	&\left(\tilde{\bbK}_j^\top \tilde{\bbK}_j+\rho \bbD_j\right)^{-1} \;\; \nonumber\\
	=&  \frac{1}{\rho} \left( \bbD_j^{-1}
		- \bbD_j^{-1}\tilde{\bbK}^\top \left(\rho\bbI+\tilde{\bbK}\bbD_j^{-1}\tilde{\bbK}_j^\top\right)^{-1}\tilde{\bbK}_j\bbD_j^{-1} \right)
\end{align}
which only requires inversion of an $M\times M$ matrix, ensuring that the computational complexity of the update~\eqref{admm:alpha1} does not grow with the network size $N$. 

In order to obtain $\bbB[k+1]$, first set $b_{ij}=0$ for all off-diagonal entries as required by the equality constraints on $\bbB$. It turns out that \eqref{admm:2} is separable over the diagonal entries $b_{jj}$, that is,
\begin{align}
\label{admm:b}
\bbB[k+1] = \underset{b_{11}, \ldots, b_{NN}}{\arg\min} \sum_{j=1}^N \|\bby_j-\tilde{\bbK}\bbalpha_{j}[k+1]-b_{jj}\bbx_j\|_2^2.
\end{align}
Per entry $j$,~\eqref{admm:b} boils down to a scalar unconstrained quadratic minimization problem with the closed-form solution
\begin{align}
\label{admm:b}
	b_{jj}[k+1]=\bbx_j^\top(\bby_j-\tilde{\bbK}\bbalpha_{j}[k+1])/{\bbx_j^\top \bbx_j}.
\end{align}
Finally,~\eqref{admm:3} can be cast as
\begin{align}
\label{admm:gmma_sub}
\bbGamma[k+1] = &\underset{\{\bbgamma_{ij}\}}{\arg \min} \;\; \sum\limits_{j=1}^N\sum_{i=1}^N \|\bbgamma_{ij}\|_2- \sum_{j=1}^N\sum_{i=1}^N\boldsymbol{\xi}_{ij}^\top[k]\bbgamma_{ij}  \nonumber\\
&+(\rho/2) \sum_{j=1}^N\sum_{i=1}^N \big \|\bbK_i^{1/2}\bbalpha_{ij}[k+1]-\bbgamma_{ij} \big \|_2^2
\end{align}
which reduces to a \emph{group Lasso} solver per index $j$ \cite{yuan2006model}. Per component vector $\bbgamma_{ij}$,~\eqref{admm:gmma_sub} can be solved in closed form via the so-termed \emph{group shrinkage} operator for each $i$ and $j$, yielding
\vspace{2mm}
\begin{align}
\label{admm:gmma}
	\bbgamma_{ij}[k]=\mathcal{P}_{\lambda/\rho}\left(\bbK_i^{1/2}\bbalpha_{ij}[k+1]+\bbxi_{ij}[k]/\rho\right)
\end{align}
which is defined as 
\vspace{2mm}
\begin{align}
\label{admm:soft_thresh}
		\mathcal{ P}_{\lambda}(\bbz):=\frac{\bbz}{\|\bbz\|_2}\max (\|\bbz\|_2-\lambda,0).
\end{align}
Together with the gradient ascent step over the dual variables, Algorithm~\ref{algo:admm} summarizes the iterations resulting from the developed ADMM solver for network topology inference task. 

Admittedly, even with the adoption of the matrix inversion lemma, ADMM incurs a considerable computational burden due to the reduced matrix inversion complexity. Nevertheless, it is important to point out that ADMM has well-documented merits in large-scale decentralized processing, where it may be feasible to split the problem across a ``coordinated'' network of computing agents; see e.g.,~\cite{schizas3,schizas4} and references therein. Due to space considerations, exploration of an ADMM-based decentralized implementation of Algorithm~\ref{algo:admm} is beyond the scope of the present paper. Instead, the sequel will focus on developing a first-order inference algorithm that leverages advances in proximal gradient descent approaches. As will be shown later, the developed approach mitigates the inherent matrix inversion complexity associated with ADMM.

\begin{algorithm}[t] 
	\caption{ADMM for kernel-based topology inference}\label{algo:admm}
	\begin{algorithmic} 
		\State\textbf{Input:}~$\bbY$, $\bbX$, $\lambda$, $\bbK$, $\tau$
		\vspace{1mm}
		\State\textbf{Initialize:}~$\bbA[0]=\boldsymbol{0}_{N\times N}$, $\bbB[0]=\boldsymbol{0}_{N\times N}$, $k=0$
		\vspace{1mm}
		\While {\text{not converged}}
			\For{$j=1,\ldots , N$ (in parallel)}
		\State Update $\tilde{\bbalpha}_j[k+1]$ via \eqref{admm:alpha1}
		
		\State Update $\bbalpha_j[k+1]$ via \eqref{admm:alpha}
		
		\State Update $b_{jj}[k+1]$ via \eqref{admm:b}
		
		\State Update $\bbgamma_{ij}[k+1]$, $\forall i\neq j$ via \eqref{admm:gmma}

		\EndFor		
		\vspace{1mm}
		\State Update $\bbW_{\alpha}[k+1]$, $\bbB[k+1]$ and $\bbGamma[k+1]$
		\State $ \boldsymbol{\Xi}[k+1]  =  \boldsymbol{\Xi}[k]+\rho (\bbD^{1/2}\bbW_{\alpha}[k+1]-\bbGamma[k+1])$
		\State $k = k + 1$
		\EndWhile
		\vspace{1mm}
		\State \underline{ Edge identification:}\\
		\vspace{1mm}
		$~\hat{a}_{ij}=1$ if $\|\bbalpha_{ij}[k]\|\geq \tau$, else $\hat{a}_{ij}[k]=0$ for all $i$,$j$ \\
		\Return $\hat{\bbA}$ and $\hat{\bbB}=\bbB[k]$		
	\end{algorithmic}
\end{algorithm}

\vspace{-4mm}
\subsection{Proximal gradient iterations}\label{sec:pg}
Since the decoupled convex cost in~\eqref{eq:prob:mod} reduces to a weighted version of the group Lasso solver, it is prudent to first define the change of variables $\boldsymbol{\zeta}_{ij}:= \bbK_i^{1/2}\bbalpha_{ij}$. Consequently,~\eqref{eq:prob:mod} can be reformulated for each node $j$ to yield
\begin{align}
\label{eq:prob:beta}
\nonumber
\{\{\hat{\bbzeta}_{ij}\}, \{ \hat{b}_{ij}\}\} &  \\
\nonumber
= \underset{ \{ \bbzeta_{ij} \}, \{b_{ij}\}}{\arg \min} 
	\frac{1}{2}& \bigg \| \bby_j-\sum_{i=1}^N\bbK_i^{1/2} \bbzeta_{ij}-\sum_{i=1}^Nb_{ij}\bbx_j \bigg \|_2^2 + \lambda \|\bbzeta_{ij}\|_2\nonumber\\
	\text{s.t.} & \;\; \bbzeta_{ij}=\boldsymbol{0}, ~~\forall i= j, \quad b_{ij}=0, ~~\forall i\neq j.
\end{align} 
Note that group-sparsity over the blocks $\{ \bbalpha_{ij} \}$ as induced by the weighted regularizer in~\eqref{eq:prob:mod} is inherited by~\eqref{eq:prob:beta}. This implies that unveiling the hidden network topology amounts to identifying which vectors in the set 
 $ \{ \bbzeta_{ij} \} $ are nonzero. Furthermore, the cost in~\eqref{eq:prob:beta} is convex and consists of the sum of differentiable and non-differentiable terms. Key to our first-order algorithm developed in this subsection is recognizing that this well-known problem structure is amenable to PG optimization tools; see e.g.,~\cite{beck2009fast} for a comprehensive review.
 
Once again temporarily ignoring the linear constraints, let 
$\bbv^j := [\bbzeta_j^\top ~b_{jj}]^\top$,  where $\bbzeta_{j}:=[\bbzeta_{1j}^\top,\ldots, \bbzeta_{Nj}^\top ]^\top$. 
Furthermore, let $\breve{\bbK} := [\bbK_1^{1/2},\ldots,\bbK_{N}^{1/2}]$, and define $f(\bbv^j):=(1/2)\|\bby_j- \breve{\bbK}\bbzeta_{j}-b_{jj}\bbx_j\|_2^2$, and $g(\bbv^j):=\lambda \sum_{i\neq j}\|\bbzeta_{ij}\|_2$. It can be shown that the gradient of $f(\bbv^j)$ is \emph{Lipschitz continuous} with (minimum) Lipschitz constant 
$ L_f :=\lambda_{\rm max} \left(\bbP_j^\top \bbP_j\right)$, where $\bbP_j:=[\breve{\bbK}~ \bbx_j]$, that is,  
$ \| \nabla f(\bbv^j_1) - \nabla f(\bbv^j_2)   \| \leq L_f \| \bbv^j_1  - \bbv^j_2\|, \; \forall \bbv^j_1, \bbv^j_2 $ in the domain of $f$, 
with $\lambda_{\rm max}(\bbZ)$ denoting the largest eigenvalue of $\bbZ$. 
Due to the Lipschitz continuity of $f(\bbv^j)$, it holds that
\vspace{2mm}
\begin{align}
\label{e	q:pg0}
\nonumber
	f(\bbv_1^j) & \leq  f(\bbv_2^j) + \langle\nabla f(\bbv_2^j),\bbv_1^j-\bbv_2^j \rangle  +  (L_f/2)  \|\bbv_1^j-\bbv_2^j\|_2^2 \\
	& =:   Q(\bbv_1^j, \bbv_2^j)
\end{align} 
where $Q(\bbv_1^j, \bbv_2^j)$ is a quadratic upper-bound of $f(\bbv_1^j)$ evaluated at $\bbv_2^j$. In general, PG algorithms judiciously select $\bbv_2^j$ from the domain of $f$, and iteratively minimize the upper bound $Q(\bbv^j, \bbv_2^j) + g(\bbv^j)$ of the cost $f(\bbv^j) + g(\bbv^j)$, until convergence is attained. Accordingly, adopting the iteration index $k$ and setting $\bbv_2^j = \bbv^j[k]$, computation of the next iterate $\bbv^j[k+1]$ amounts to solving (cf. \cite{beck2009fast})
\vspace{1mm}
\begin{align}
\label{eq:pg1}
\nonumber
\bbv^j[k+1]  :=&~  \underset{\bbv^j}{\text{arg min }} Q(\bbv^j,{\bbv}^j[k]) + g(\bbv^j) \\
		=&~  \underset{\bbv^j}{\text{arg min }} (L_f/2) \| \bbv^j - \mathbf{u}^j[k] \|_2^2 + g(\bbv^j).
\end{align}
It turns out that $\mathbf{u}^j[k] := \bbv^j[k] - (1/L_f) \nabla f(\bbv^j[k])$ is an ordinary gradient descent step evaluated at $\bbv^j[k]$. As will be shown next, evaluation of the \emph{proximal operator} $\text{arg} ~\underset{\bbv^j}{\text{min}} (L_f/2) \| \bbv^j - \mathbf{u}^j[k] \|_F^2 + g(\bbv^j) $ is straightforward, rendering this iterative approach quite attractive. In fact, adoption of iterative PG solvers for most regularized estimators is well motivated because $\bbv^j[k+1]$ is available in closed form. 

In order to incorporate the linear equality constraints, let $\breve{\bbK}_{s}$ denote the $M \times M(N-1)$ matrix obtained by removing all columns indexed by $\mathcal{I}_j$ from $\breve{\bbK}$, and $\tilde{\bbzeta}_j:=[\bbzeta_{1j}^\top, \ldots, \bbzeta_{(j-1)j}^\top, \bbzeta_{(j+1) j}^\top, \ldots, \bbzeta_{Nj}^\top]^{\top}$. Modifying the cost function accordingly, $f(\bbv^j)$ can then be written as
\begin{align}
\label{eq:pg2}
	f(\bbv^j):=(1/2)\|\bby_j-\breve{\bbK}_{j}\tilde{\bbzeta}_{j}-b_{jj}\bbx_j\|_2^2.
\end{align}
Clearly, $ \nabla f(\bbv^j)$ decouples into gradients with respect to $\tilde{\bbzeta}_{j}$ and $b_{jj}$ to obtain
\begin{subequations}
\begin{equation}
\label{ista:grad:a}
    \nabla_{\tilde{\bbzeta}_{j}} f(\bbv^j[k]) =\breve{ \bbK}_{j}^\top \left( \breve{\bbK}_{j} \tilde{\bbzeta}_{j}[k]   + b_{jj}[k] \bbx_j  - \bby_j \right)   	
\end{equation}
\begin{equation}
	\label{ista:grad:b}
		\nabla_{b_{jj}} f(\bbv^j[k]) = \left( \breve{ \bbK}_{j} \tilde{\bbzeta}_{j}[k] + b_{jj}[k] \bbx_j  - \bby_j \right)^{\top} \bbx_j 
\end{equation}
\end{subequations}
which upon substitution into~\eqref{eq:pg1}, yields the following provably-convergent PG iterations per node $j$
\begin{subequations}
\begin{eqnarray}
\label{ista:update:z}
\bbz_{j}[k] &=& \tilde{\bbzeta}_{j}[k] - (1/L_f) \nabla_{\tilde{\bbzeta}_{j}}f(\bbv^j[k]) \\
\label{ista:update:a}
\bbzeta_{ij}[k+1] &=& \mathcal{P}_{\lambda/L_f}\left(\bbz_{ij}[k]\right)\\
\label{ista:update:b}
b_{jj}[k+1] &=& b_{jj}[k] - (1/L_f) \nabla_{b_{jj}}f(\bbv^j[k]).
\end{eqnarray}
\end{subequations}
In step~\eqref{ista:update:a}, $\bbz_{ij}$ denotes a subvector of $\bbz$ with entries indexed by $\mathcal{I}_j$, while the projection operator $\mathcal{P}_{\lambda/L_f}(.)$ has been previously defined in~\eqref{admm:soft_thresh}. The sought estimate of $\bbzeta_j[k+1]$ is now obtained by stacking and zero-padding $\bbzeta_{ij}[k+1]$ as
\begin{multline}
\label{pg:zeta}
	\bbzeta_j[k]:=\big[ \bbzeta_{1j}^\top[k+1], \ldots, \bbzeta_{(j-1)j}^\top[k+1], \\ 
	\boldsymbol{0}^{\top}_{1\times M}, \bbzeta_{(j+1) j}^\top[k+1],\ldots,\bbzeta_{Nj}^\top[k+1] \big]^\top.
\end{multline}
Entries of the adjacency matrix $\bbA$ are then set by identifying the nonzero vectors $\bbzeta_{ij}$, while entries of the diagonal matrix $\bbB[k]$ are selected from the estimates of $\{ b_{ii}[k] \}_{i=1}^N$ associated with the appropriate indices. Algorithm~\ref{algo1} summarizes all the steps constituting the developed PG algorithm.

\begin{algorithm}[t] 
	\caption{Kernel-based PG inference algorithm}\label{algo1}
	\begin{algorithmic} 
		\State\textbf{Input:}~$\bbY$, $\bbX$, $\bbK$, $\tau$, $\lambda$
		
		\State\textbf{Initialize:}~$\bbA[0]=\boldsymbol{0}_{N\times N}$, $\bbB[0]=\boldsymbol{0}_{N\times N}$, $k=0$
		
		\While {\text{not converged}}
			\For{$j=1,\ldots , N$ (in parallel)}
		
		\State {\bf(S1)} \underline{Gradient calculation:}\\ 
		\hspace{1.8cm}Calculate gradients at $\tilde{\bbzeta}_j[k]$ via \eqref{ista:grad:a}\\
	    \hspace{1.8cm}Calculate gradients at $b_{jj}[k]$ via \eqref{ista:grad:b}
	    
		\State {\bf(S2)} \underline{Variable updates:}\\
		\hspace{1.8cm}Update $\bbz_j[k+1]$ via \eqref{ista:update:z}\\
		\hspace{1.8cm}Update $\bbzeta_{ij}[k+1]$, $\forall ~ j\neq i$ via \eqref{ista:update:a}\\
		\hspace{1.8cm}Update $b_{jj}[k+1]$ via \eqref{ista:update:b}\\
		\hspace{1.8cm}Update $\bbzeta_{j}[k+1]$ via \eqref{pg:zeta}

		\EndFor		
		\State $k = k + 1$
		\EndWhile
		\State \underline{Edge identification:}\\
		Set $\hat{a}_{ij}=1$ if $\|\bbzeta_{ij}[k]\|\geq \tau$, else $\hat{a}_{ij}[k]=0$ for all $i,j$\\
		\Return $\hat{\bbA}$ and $\hat{\bbB}=\bbB[k]$
			
	\end{algorithmic}
\end{algorithm}

\begin{algorithm}[t] 
	\caption{Kernel-based APG inference algorithm}\label{algo:fista}
	\begin{algorithmic} 
		\State\textbf{input:}~$\bbY$, $\bbX$, $P$, $\lambda$, $\tau$
		
		\State\textbf{initialize:}~$\bbA[0]=\boldsymbol{0}_{N\times NP}$, $\bbB[0]=\boldsymbol{0}_{N\times NP}$
		
		\For {$k = 1, 2,\ldots, K$}
			\For{$j=1,\cdots N$ (in parallel)}
		
		\State {\bf(S1)} \underline{Intermediate variable updates:}\\
		\hspace{1.8cm}Update $\bbz_j[k]$ via \eqref{fista:z}  \\
		\hspace{1.8cm}Update $d_{jj}[k]$ via \eqref{fista:d}
		
		\State {\bf(S2)} \underline{Gradient calculation:}\\ 
		\hspace{1.8cm}Calculate gradients at $\tilde{\bbz}_{j}[k]$ via \eqref{fista:grad:a}\\
	    \hspace{1.8cm}Calculate gradients at $d_{jj}[k]$ via \eqref{fista:grad:b}
		
		\State {\bf(S3)} \underline{Variables updates:}\\
		\hspace{1.8cm}Update $\tilde{\bbz}_j[k+1]$ via \eqref{fista:update:z}\\
		\hspace{1.8cm}Update $\bbzeta_{ij}[k+1]$, $\forall ~ j\neq i$ via \eqref{fista:update:a}\\
		\hspace{1.8cm}Update $b_{jj}[k+1]$ via \eqref{fista:update:b}\\
		\hspace{1.8cm}Update $\bbzeta_j[k+1]$ via \eqref{pg:zeta}

		\EndFor		
		\EndFor
		\State \underline{Edge identification:}\\
		   Set $~\hat{a}_{ij}=1$ if $\|\bbzeta_{ij}[k]\|\geq \tau$, else $\hat{a}_{ij}[k]=0$ for all $i, j$\\
		\Return $\hat{\bbA}$ and $\hat{\bbB}=\bbB[k]$
		
	\end{algorithmic}
\end{algorithm}

Although Algorithm~\ref{algo1} is computationally simple and provably convergent, it has been shown to converge slowly at times, and several \emph{accelerated} variants have been developed. Prominent among these are the so-termed \emph{accelerated proximal gradient (APG)} approaches that offer a worst-case convergence guarantee of $\mathcal{O}(1/\sqrt{\epsilon})$ iterations for an $\epsilon$-optimal solution [cf. $\mathcal{O}(1/{\epsilon})$ for PG methods]; see e.g.,~\cite{nesterov2005smooth} for details. The key to this convergence improvement lies in computing $\mathbf{u}^j[k]$ in~\eqref{eq:pg1} using a linear combination of the two most recent iterates, $\bbv^j[k-1]$ and $\bbv^j[k-2]$. Adopting this strategy, convergence improvement of Algorithm~\ref{algo1} is straightforward, and amounts to the following modifications [cf. \eqref{eq:pg1}]
\begin{align}
\bbv^j[k+1] =& \arg\min_{\bbv^j} \;\; Q(\bbv^j,\bbu^j[k])\nonumber\\
		=&\arg\min_{\bbv^j} \;\; (L_f/2) \big \| \bbv^j - \big( \breve{\bbu}^j[k] \nonumber\\ 
		&\hspace{1.0cm} - (1/L_f) \nabla f(\breve{\bbu}^j[k]) \big) \big \|_2^2 + g(\bbv^j)
\end{align}
where 
\begin{align}
	\breve{\bbu}^j[k]&:=\bbv^j[k-1]+\left( \frac{\beta_{k-1}-1}{\beta_k}\right) \left(\bv^j[k-1]-\bv^j[k-2]\right)\\
	\beta_k& =\left(1+\sqrt{4\beta_{k-1}^2+1} \right)/2.
	\label{fista:beta}
\end{align}
Similar to \eqref{ista:update:z}-\eqref{ista:update:b}, it turns out that the resulting APG updates can be written as
\begin{subequations}
	\begin{align}
	\label{fista:z}
		\bbz_{j}[k] & = \bbzeta_j[k]+\left( \frac{\beta_{k-1}-1}{\beta_k} \right)(\bbzeta_j[k]-\bbzeta_j[k-1])\\
	\label{fista:d}
	d_{jj}[k+1]& =  b_{jj}[k]+\left( \frac{\beta_{k-1}-1}{\beta_k} \right) ({b_{jj}}[k]-{b_{jj}}[k-1])\\
	\label{fista:update:z}
	\tilde{\bbz}_j[k]&= 	\bbz_{j}[k]-(1/L_f)\nabla_{\tilde{\bbz}_{j}}f[k]\\	
	\label{fista:update:a}
	\bbzeta_{ij}[k+1]& =\mathcal{P}_{\lambda/L_f}\left(\tilde{\bbz}_{ij}[k]\right)\\
	\label{fista:update:b}
	b_{jj}[k+1]& = d_{jj}[k]-(1/L_f)\nabla_{d_{jj}}f[k].
	\end{align}
	\end{subequations}
\noindent Clearly, $ \nabla f(\bbv)$ decouples into gradients with respect to $\bbzeta_{j}$ and $b_{ii}$, that is,
\begin{subequations}
\begin{align}
	\label{fista:grad:a}
    \nabla_{\tilde{\bbz}_{j}} f(\bbv^j[k])&= \bbK_{s,j}^\top \left( \bbK_{s,j} \tilde{\bbz}_{j}[k]   +d_{ii}[k] \bbx_i  
     - \bby_j \right) 	\\
	\label{fista:grad:b}
		\nabla_{d_{ii}} f(\bbv^j[k])&= \left( \bbK_{s,j}\tilde{\bbz}_{j} [k]  
		+b_{ii}[k] \bbx_j
		-\bby_j  \right)^{\top}\bbx_j.
\end{align}
\end{subequations}
With a similar update for  $\bbzeta_j[k]$ as highlighted in \eqref{pg:zeta}, the developed accelerated variant of the PG solver is summarized in Algorithm \ref{algo:fista}.

\vspace{3mm}
\section{Special case: Polynomial SEM}
 \label{sec:polysem}
The developed kernel-based topology inference approach is quite general, and eliminates the need to explicitly specify the functions $\{ \phi_p(.) \}_{p=1}^P$ a priori. Nevertheless, if $\phi_p(.)$ is known to belong to a specific family, e.g. polynomial functions, it may be possible to derive more efficient estimators that capitalize on such prior knowledge. Polynomial SEMs have been advocated in e.g.,~\cite{harring2012comparison,kelava2014nonlinear}, with the underlying network topology assumed known a priori. This section does not presume such prior knowledge, and postulates that $\psi_{\text{en}, j} (y_{jm})$ admits a polynomial expansion whose coefficients capture the unknown topology; that is,
\begin{equation}
\label{eq:mod:poly}
	 \psi_{\text{en}, ij} (y_{im})  =  \sum\limits_{p = 1}^P   c_{ijp} y_{im}^p 
\end{equation}
with the summand $\phi_p(y_{im})=y_{im}^p$. Defining $\tilde{\bby}_{im}:=[y_{im},y_{im}^2,\ldots,y_{im}^P]^\top$,~\eqref{eq:mod:poly} can be equivalently written as $\psi_{\text{en}, ij} (y_{im}) = \bbc_{ij}^\top \tilde{\bby}_{im}$, which is clearly a special case of the endogenous term in~\eqref{mod:nonlinear}, with $\bbphi(y_{im})=\tilde{\bby}_{im}$. More specifically, adopting the polynomial endogenous term~\eqref{eq:mod:poly}, one obtains 
 \begin{align}
 	{\bby}_j=\sum_{i\neq j}\tilde{\bbY}_{i}\bbc_{ij}a_{ij}+b_{jj}\bbx_j+\bbe_j
 \end{align} 
where $\tilde{\bbY}_{i} { := }[\tilde{\bby}_{i1},\ldots,\tilde{\bby}_{iM}]^\top$, with all other variables previously defined. Furthermore, with $\bbY:=[\bby_1, \dots,\bby_N]$, $\bbX:=[\bbx_1,\dots,\bbx_N]$, and $\tilde{\bbY}:=[\tilde{\bbY}_1, \dots,\tilde{\bbY}_N]$, one obtains the corresponding nonlinear matrix SEM
\begin{align}
	\bbY=\tilde{\bbY}\bbW+\bbX\bbB+\bbE
	\label{mod:nonlinear_pol}
\end{align}
where, as earlier observed, $\bbW$ exhibits a block structure dictated by the hidden network topology. By similarly adopting a regularized LS approach to estimate the unknown matrices in~\eqref{mod:nonlinear_pol}, the following estimator is advocated
\begin{align}
\label{eq:obj}
    \nonumber
	\{\hat{\bbW},\hat{\bbB}\} = \underset{\bbW, \bbB}{\arg\min}  & \;\; (1/2) \|\bbY-\tilde{\bbY}\bbW-\bbX\bbB\|_F^2+\lambda \|\bbW\|_{\mathcal{I},1}\\
	\text{s.t.} &  \;\; \bbw_{jj}=\boldsymbol{0} \; \forall j,  \; b_{ij}=0 \;\; \forall i\neq j
\end{align} 
where $\bbw_{ij}:= a_{ij}\bbc_{ij}$, while the constraints are once again enforced to ensure absence of self-loops, and that $\bbB$ is a diagonal matrix. Since edge sparsity is reflected in the block-sparsity of $\bbW$, the well-known group-sparse regularizer is adopted as an additive penalty defined as
\begin{align}
	 \|\bbW\|_{\mathcal{I},1}:=\sum_{i=1}^N\sum_{j=1}^N\left( \sum_{k\in \mathcal{I}_{j}}w_{ik}^2 \right)^{1/2}=\sum_{i,j}\|\bbw_{ij}\|_2.
\end{align}
As earlier alluded to in a similar context, $\hat{\bbW}$ contains sufficient information to recover the adjacency matrix $\mathbf{A}$, namely by identifying the non-zero blocks. Recognizing that the cost function in~\eqref{eq:obj} retains the properties of estimators put forth in the prequel (convex with a non-differentiable, additive term), proximal gradient iterations will be adopted next to estimate the unknowns.

Let  $\tilde{\bbw}_{j}$ denote the $ (N-1)P\times 1$ vector obtained by removing entries indexed by $\mathcal{I}^{'}_j:=\{jP-P+1,\ldots,jP\}$ from the $j$th column of $\bbW$. Similarly, let $\tilde{\bbY}_{-i}$ denote the $(N-1)P\times M$ matrix obtained by removing all columns indexed by $\mathcal{I}^{'}_i$ from $\tilde{\bbY}$. Modifying the cost function to incorporate the linear equality constraints, and defining $\bv^j:= [\tilde{\bbw}_j^\top~ \bbx_j^\top]^\top$, it turns out that 
\begin{align}
\label{eq:pg2}
	f(\bv^j):= (1/2) \big \|\bby_j-\tilde{\bbY}_{-j}\tilde{\bbw}_{j}-b_{jj}\bbx_j \big \|_2^2.
\end{align}
Leveraging the PG iterations from~\eqref{eq:pg1}, $ \nabla f(\bv^j[k])$ decouples into gradients with respect to $\bbw_{j}$ and $b_{jj}$, that is,
\begin{subequations}
\begin{equation}
\label{poly_ista:grad:a2}
    \nabla_{\tilde{\bbw}_{j}} f(\bbv^j[k]) = \tilde{\bbY}_{j}^\top \left(\tilde{\bbY}_{j} \tilde{\bbw}_{j}[k]  + b_{jj}[k]\bbx_j - \bby_j \right) 
\end{equation}
\begin{equation}
	\label{poly_ista:grad:b2}
		\nabla_{b_{jj}} f(\bbv^j[k]) =\left(\tilde{\bbY}_{j} \tilde{\bbw}_{j}[k]  + b_{jj}[k]\bbx_j - \bby_j \right)^\top \bbx_j
\end{equation}
\end{subequations}
which leads to the following steps per iteration $k$
\begin{subequations}
\begin{eqnarray}
\label{ista:update:z2}
\bbz_{j}[k] &=& \tilde{\bbw}_{j}[k] -(1/L_f) \nabla_{\tilde{\bbw}_{j}}f(\bv^j[k]) \\
\label{ista:update:a2}
\bbw_{ij}[k+1] &=& \mathcal{P}_{\lambda/L_f}\left(\left[\bbz_{j}[k] \right]_{\mathcal{I}_j}\right)\\
\label{ista:update:b2}
b_{ii}[k+1] &=& b_{jj}[k] - (1/L_f) \nabla_{b_{jj}}f(\bv^j[k])
\end{eqnarray}
\end{subequations}
where $\left[ \bbz \right]_{\mathcal{I}_j}$ denotes a subvector of $\bbz$ whose entries comprise coefficients indexed by elements of $\mathcal{I}_j$. Rows of $\hat{\bbW}$ are similarly obtained by stacking and zero-padding [cf.~\eqref{pg:zeta}]
\begin{multline}
		\bbw_j[k]=\big[\bbw_{1j}^\top[k],\cdots, \bbw_{(i-1)j}^\top[k], \\ 
		\boldsymbol{0}_{1\times P}^{\top}, \bbw_{(i+1)j}^\top[k],\cdots,\bbw_{Nj}^\top[k]\big]^\top.
\label{ista:a2_padding}
\end{multline}
Algorithm~\ref{algo:ista} summarizes the developed PG solver for the nonlinear polynomial SEM. Acceleration of Algorithm~\ref{algo:ista} for faster convergence is straightforward, and follows along similar modifications to those highlighted in the prequel.

\begin{algorithm}[t] 
	\caption{PG solver for polynomial SEM}\label{algo:ista}
	\begin{algorithmic} 
		\State\textbf{Input:}~$\bbY$, $\bbX$, $\lambda$, $P$, $\tau$
		
		\State\textbf{Initialize:}~$\bbW[0]=\boldsymbol{0}_{N\times NP}$, $\bbB[0]=\boldsymbol{0}_{N\times NP}$, $k=0$
		
		\While {\text{not converged}}
			\For{$j=1,\ldots , N$ (in parallel)}
					
		\State {\bf(S1)} \underline{Gradient calculation:}\\ 
		\hspace{1.8cm}Calculate gradients at $\tilde{\bbw}_j[k]$ via \eqref{poly_ista:grad:a2}\\
	    \hspace{1.8cm}Calculate gradients at $b_{jj}[k]$ via \eqref{poly_ista:grad:b2}
		
		\State {\bf(S2)} \underline{Variable updates:}\\
		\hspace{1.8cm}Update $\bbz_j[k+1]$ via \eqref{ista:update:z2}\\
		\hspace{1.8cm}Update $\bbw_{ij}[k+1]$, $\forall ~ j\neq i$ via \eqref{ista:update:a2}\\
		\hspace{1.8cm}Update $b_{jj}[k+1]$ via \eqref{ista:update:b2}\\
		\hspace{1.8cm}Update $\bbw_j[k+1]$ via \eqref{ista:a2_padding} 

		\EndFor		
		\State Update $\bbW[k+1]$ and $\bbB[k+1]$
		\State $k = k + 1$
		\EndWhile
		\State \underline{Edge identification:}\\
		   Set $~\hat{a}_{ij}=1$ if $\|\bbw_{ij}[k]\|\geq \tau$, else $\hat{a}_{ij}[k]=0$ for all $i, j$\\
		\Return $\hat{\bbA}$ and $\hat{\bbB}=\bbB[k]$
		
	\end{algorithmic}
\end{algorithm}
%
\begin{remark}
It turns out that polynomial regression is a special case of kernel regression, with an appropriate choice of polynomial kernel; see e.g.,~\cite{franz2006unifying}. Not surprisingly, a similar link can be drawn for the kernel-based SEM advocated in Section \ref{sec:kernels}. Specifically, upon adoption of the kernel function $k(x,y)=\sum_{p=1}^P (xy)^p$,~\eqref{eq:prob:mod} can be shown equivalent to \eqref{eq:obj}. Note that even when the form of $ \psi_{\text{en}, ij} (y_{im})$ is explicitly known, Algorithms~\ref{algo:admm} and~\ref{algo1} can still be run upon computation of the kernel matrices from nodal measurements $\{y_{im}\}$. However, comparing \eqref{ista:grad:a} and \eqref{ista:grad:b} with \eqref{poly_ista:grad:a2} and \eqref{poly_ista:grad:b2} reveals that when $P\ll M$, one incurs lower computational cost by explicitly modeling node dependencies via a polynomial SEM than the more general kernel-based approach. This is due to the smaller dimensions involved in the constituent matrix products.
\end{remark}

\vspace{1mm}
\section {Numerical Tests}
\label{sec:test}

\subsection{Tests on simulated data}
This subsection presents results of numerical tests conducted on synthetic data to assess the effectiveness of the developed algorithms.

\noindent\textbf{Data generation.} A Kronecker random graph with $N=64$ nodes was generated with a ``seed matrix'' 
\begin{align}
\bbS_0:=\left(
\begin{array}{cccc}
0 &	0&1&1\\
0 &	0&1&1\\
0 &	1&0&1\\
1 &	0&1&0
\end{array} \right) \nonumber
\end{align}
 to obtain a binary-valued $64\times 64$ matrix via repeated Kronecker products, namely $\bbS=\bbS_0\otimes\bbS_0\otimes\bbS_0$; see also \cite{leskovec2010kronecker}. The Kronecker graph with adjacency matrix $\bbA$ was then constructed by randomly sampling each entry from a Bernoulli distribution with $a_{ij}\sim\text{Bernoulli}(0.3s_{ij})$. 
 
With different values of $M$, entries of $\bbY\in\mathbb{R}^{N\times M}$ were randomly sampled from the standard normal distribution $(y_{nm} \sim \mathcal{N}(0,1))$. Kernel matrices $\{ \bbK_m \}$ were generated using prescribed kernels, that is, entry $(i,j)$ of $\bbK_m$ was set to $[\bbK_m]_{ij}=k(y_{im}, y_{jm})$, where the kernel function $k(\cdot, \cdot)$ is known a priori. Matrix $\bbB\in \mathbb{R}^{N\times N}$ was constructed as a diagonal matrix with entries drawn from an i.i.d. standardized normal distribution. Entries of coefficient vectors $\bbalpha_{ij}\in\mathbb{R}^M $ were drawn uniformly from the interval $[-0.2, 0.2]$, while noise terms were generated i.i.d. as $e_{ij}\sim \calN (0,\sigma_e^2)$, with $\sigma_e=0.01$. Finally, the exogenous matrix $\bbX$ was generated from other terms using the postulated nonlinear SEM, that is, $\bbX=(\bbY-\tilde{\bbK} \bbW_{\alpha}-\bbE)\bbB^{-1}$, where $\bbW_{\alpha}$ was constructed with the $(i, j)$-th block set to $\bbalpha_{ij}$. 

Experiments were run for different values of $M$, with thresholds that control presence or absence of an edge (denoted by $\tau$ in the listed algorithms) selected to obtain the best edge identification accuracy. Furthermore, sparsity-promoting regularization parameters ($\lambda$) were all judiciously selected to obtain the lowest edge identification  error rate (EIER), defined as
\begin{align}
	\text{EIER}:=\frac{\|\bbA-\hat{\bbA}\|_0}{N(N-1)}\times 100\%
\end{align}
with the operator $\|\cdot\|_0$ denoting the number of nonzero entries of its argument. For all experiments, error plots were generated using values of EIER averaged over $100$ independent runs. 
\begin{figure}[tpb!]
	\centering
	\includegraphics[scale=0.6]{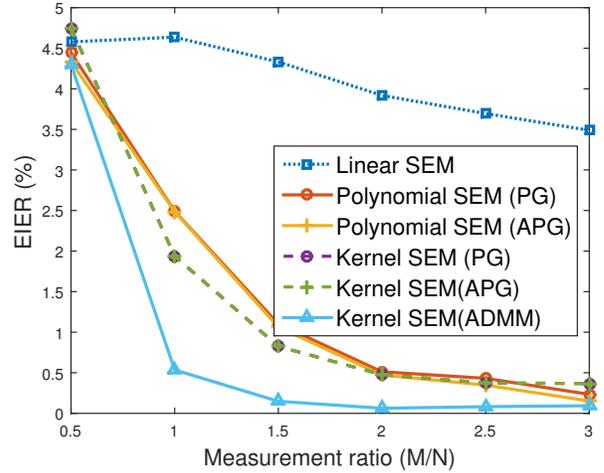}
\caption{EIER plotted against the measurement ratio (M/N) with simulated data generated via a polynomial kernel of order $P=2$. The nonlinear topology identification approaches markedly outperform the linear SEM.}\label{fig:er-poly2}
\end{figure}
%
%
%
%

%
%

%
\begin{figure}[tpb!]
	\centering
	\includegraphics[scale=0.6]{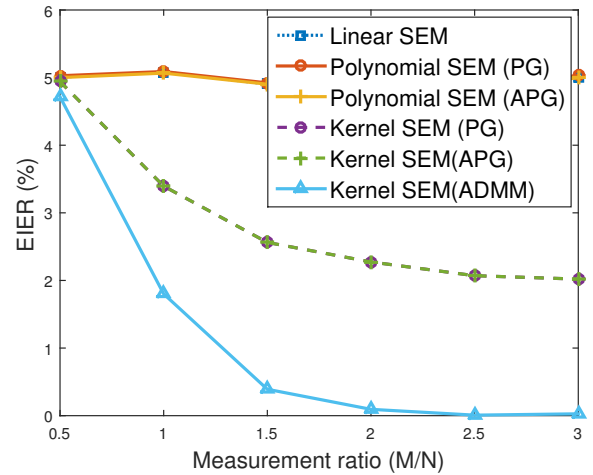}
\caption{EIER plotted against the measurement ratio with data generated using a Gaussian kernel with $\sigma^2=0.01$. In this case, adopting the polynomial SEM does not lead to a noticeable performance improvement over the linear SEM. Nevertheless, all kernel-based approaches markedly outperform the linear SEM.}
\label{fig:er-rbf01}
\end{figure}

\noindent\textbf{Test results.} 
Figures~\ref{fig:er-poly2} and~\ref{fig:er-rbf01} depict plots of the EIER against the measurement ratio ($M/N$), for polynomial and Gaussian kernels, respectively. Specifically, Figure~\ref{fig:er-poly2} plots the EIER when data are generated by~\eqref{mod:nonlinear}, using a polynomial kernel of order $P=2$. Clearly, adoption of nonlinear SEMs yields markedly better performance than topology inference approaches based on linear SEMs, as demonstrated by the lower EIER values. It turns out that running the ADMM solver consistently leads to lower EIER than its (A)PG counterparts. This performance gap can be attributed to parameter transformation for the latter solvers, that is, presence or absence of an edge was based on $\bbzeta_{ij}$, a linear transformation of $\bbalpha_{ij}$. 

Figure~\ref{fig:er-poly2} also depicts test results for experiments conducted with \emph{a fortiori} selection of the polynomial SEM, advocated in Section~\ref{sec:polysem}. Setting the polynomial order to $P=2$, the kernel-based solvers are observed to yield improved performance over the polynomial SEM. Although one can readily attribute this to the inherent model mismatch, it is worth noting that even solvers based on the polynomial SEM markedly outperform the linear SEM.

Figure~\ref{fig:er-rbf01} plots the EIER curves for the setting in which data were generated via Gaussian kernels $\kappa(x,y) := e^{-(x-y)^2/2\sigma^2}$, with the bandwidth parameter set to $\sigma^2=0.01$. The plot also depicts the EIER curves resulting from adopting a polynomial SEM of order $P=3$ for the same data. Not surprisingly, kernel-based approaches significantly outperform the conventional linear SEM. In this setting, adopting the polynomial SEM does not lead to better performance than the linear SEM, as was the case with the data generated via polynomial kernels. 

On the other hand, polynomial SEMs are attractive due to their reduced computational complexity, leading to algorithms that are on average three times  faster than the more general kernel-based alternatives. It is also worth pointing out that increasing the measurement ratio yields lower EIER, especially for the kernel-based algorithms. Figure~\ref{fig:net} depicts plots of actual and inferred adjacency matrices, under varying modeling assumptions. Plots of inferred adjacency matrices are based on a single realization of $M=128$ samples, and non-zero entries are illustrated in white. As demonstrated by the plots, accounting for nonlinearities yields more accurate recovery of the unknown network topology.
 
%
\begin{figure}[tpb!]
\begin{minipage}[b]{.24\textwidth}
\centering
\includegraphics[width=4.3cm]{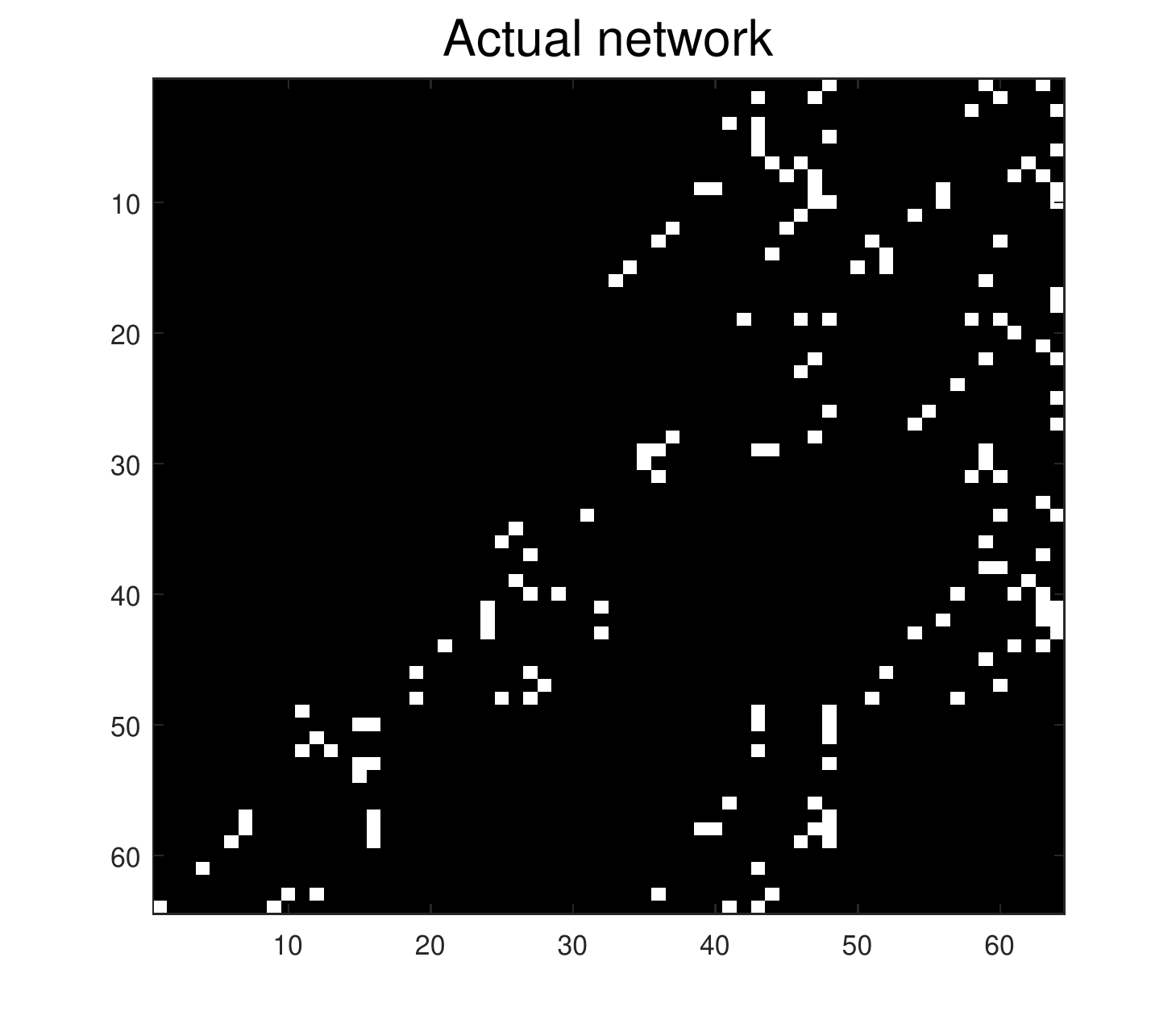}
\end{minipage}
\begin{minipage}[b]{.24\textwidth}
\centering
\includegraphics[width=4.3cm]{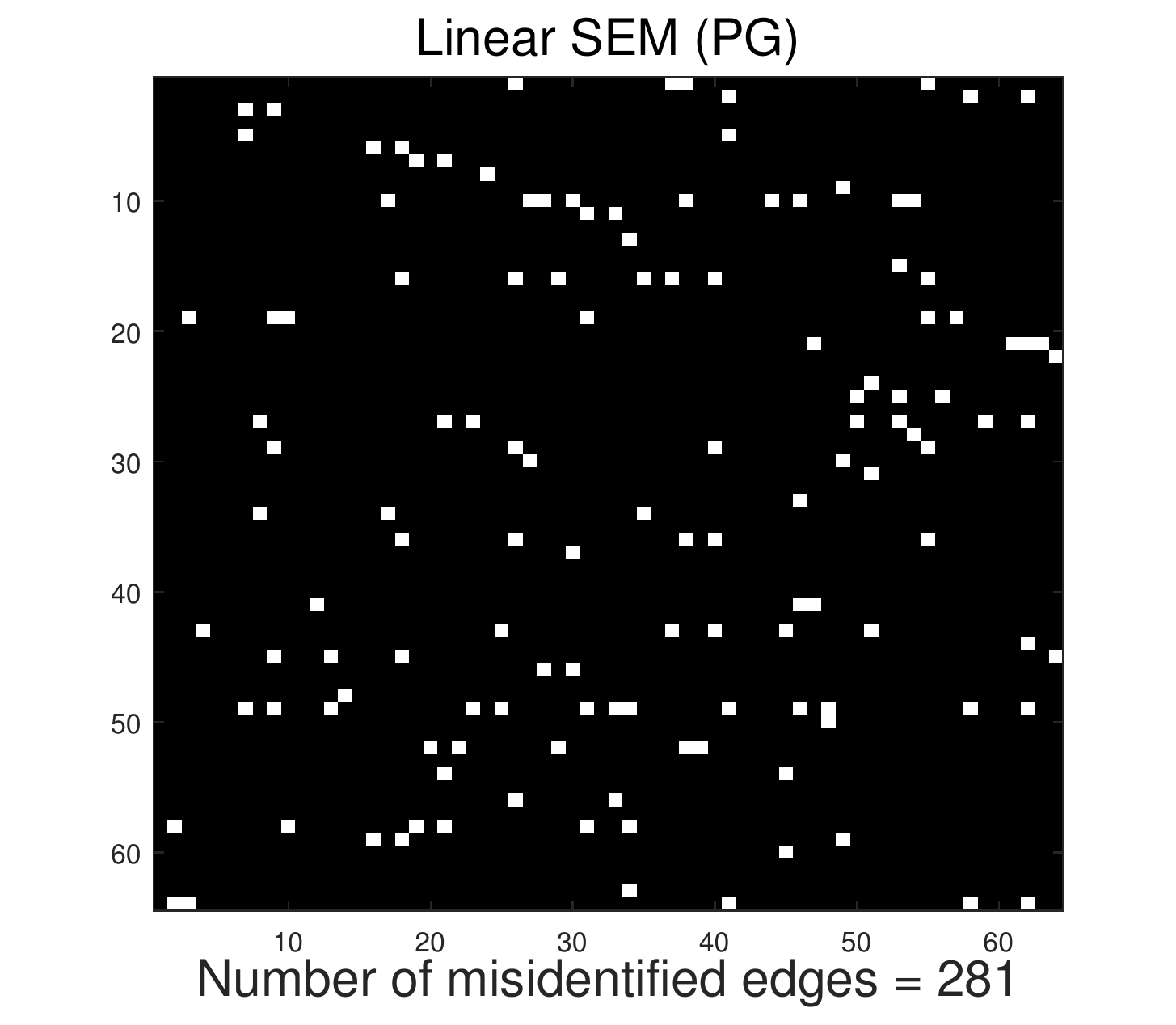}
\end{minipage}
\begin{minipage}[b]{.24\textwidth}
\centering
\includegraphics[width=4.3cm]{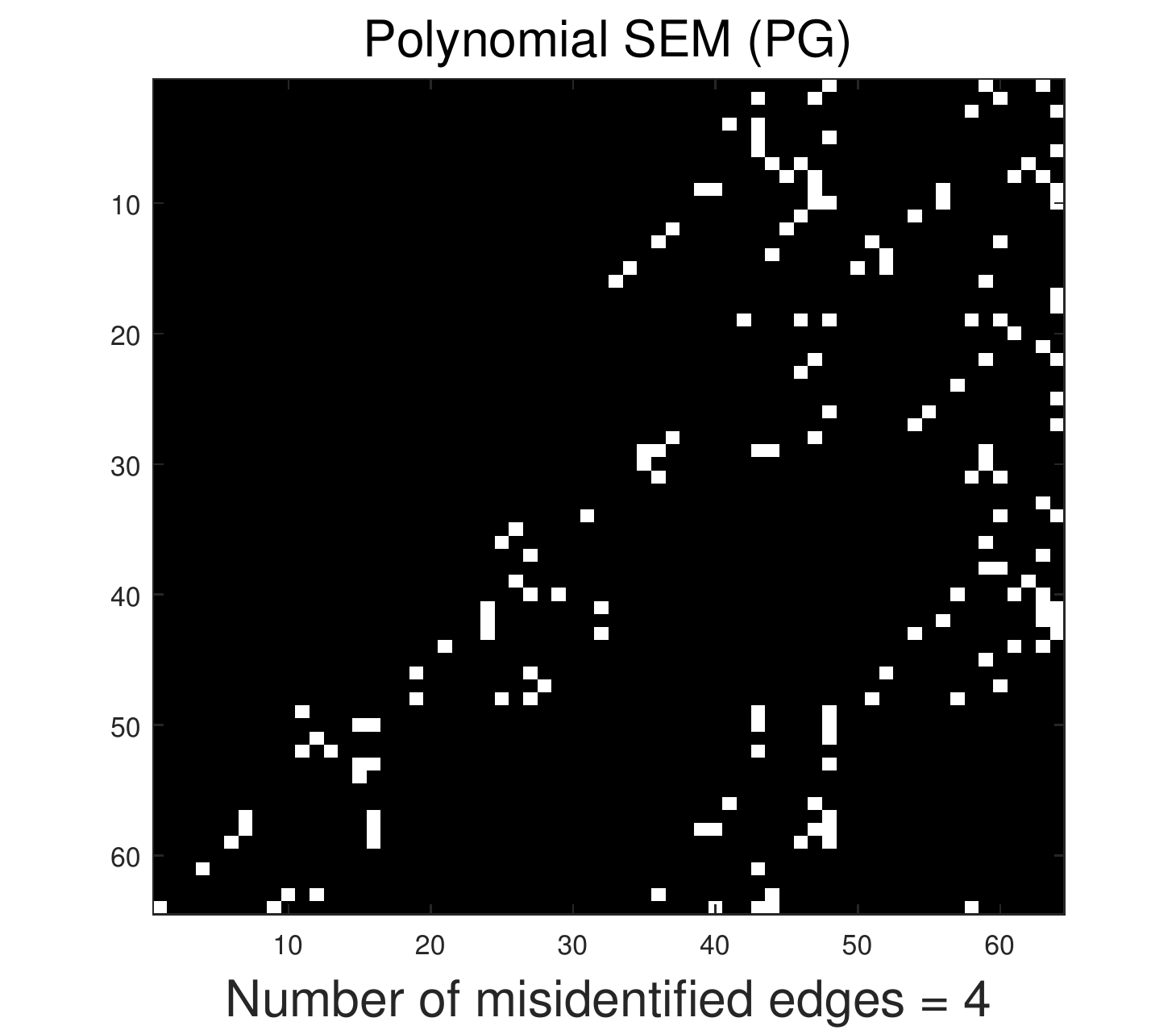}
\end{minipage}
\begin{minipage}[b]{.24\textwidth}
\centering
\includegraphics[width=4.3cm]{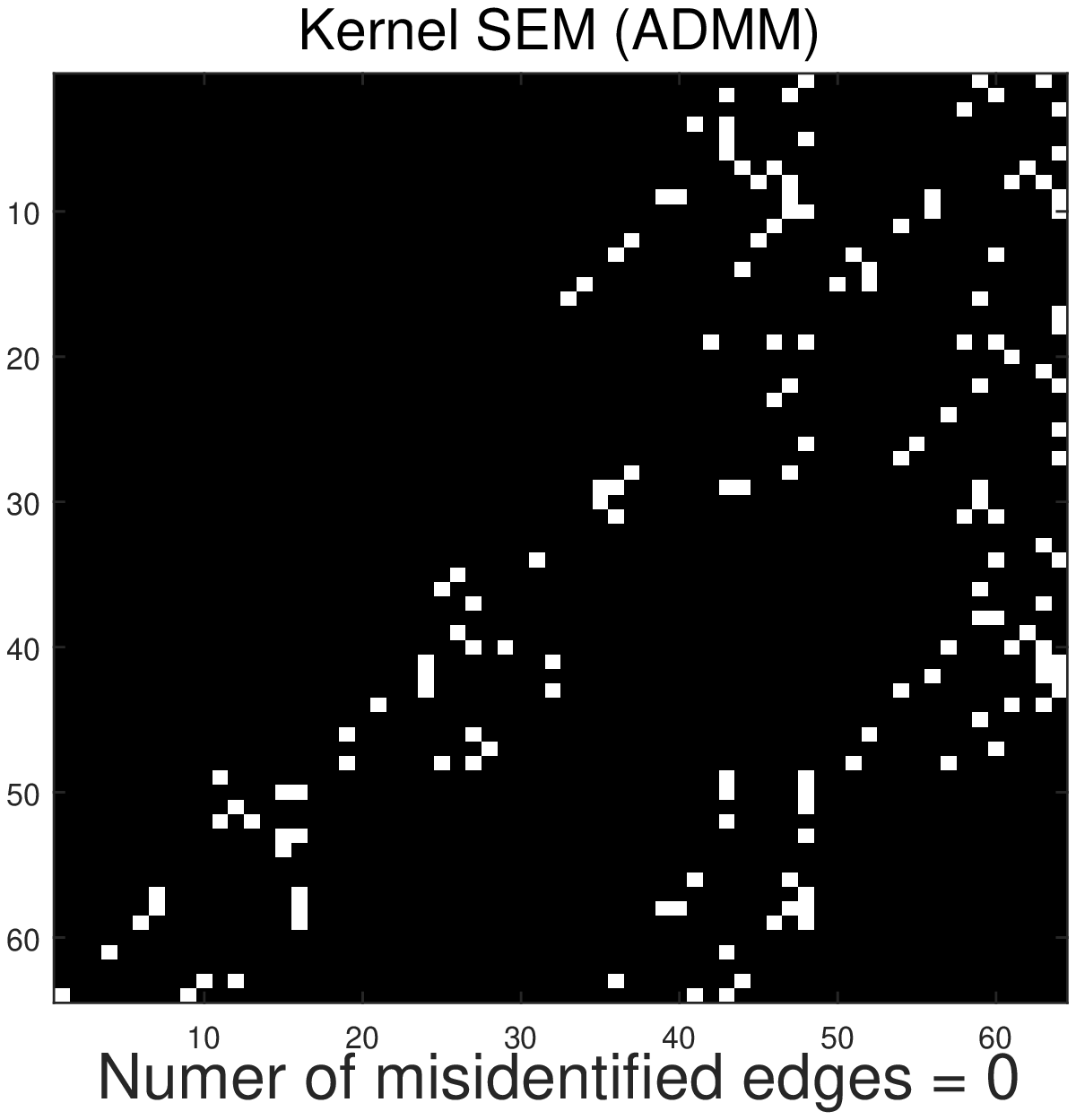}
\end{minipage}
 \caption{Plots of actual and inferred adjacency matrices resulting from adopting linear and nonlinear SEMs with $M=128$.} 
 \label{fig:net}
\end{figure}
\begin{figure*}[tpb!]
\hspace{-3.5mm}
\begin{minipage}[b]{.33\textwidth}
\centering
\includegraphics[width=6.7cm]{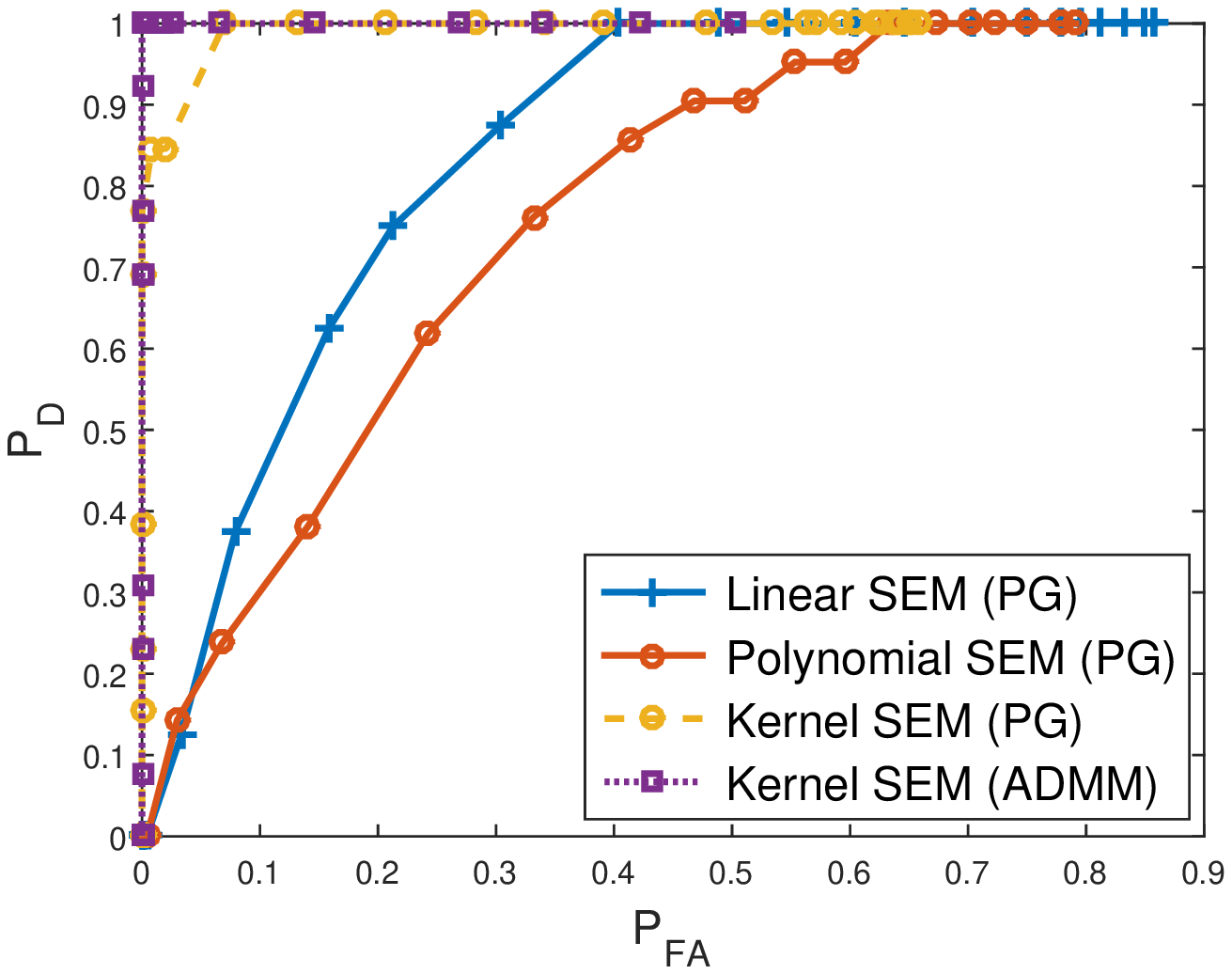}
\centerline{(a)}
\end{minipage}
\begin{minipage}[b]{.33\textwidth}
\centering
\includegraphics[width=6.7cm]{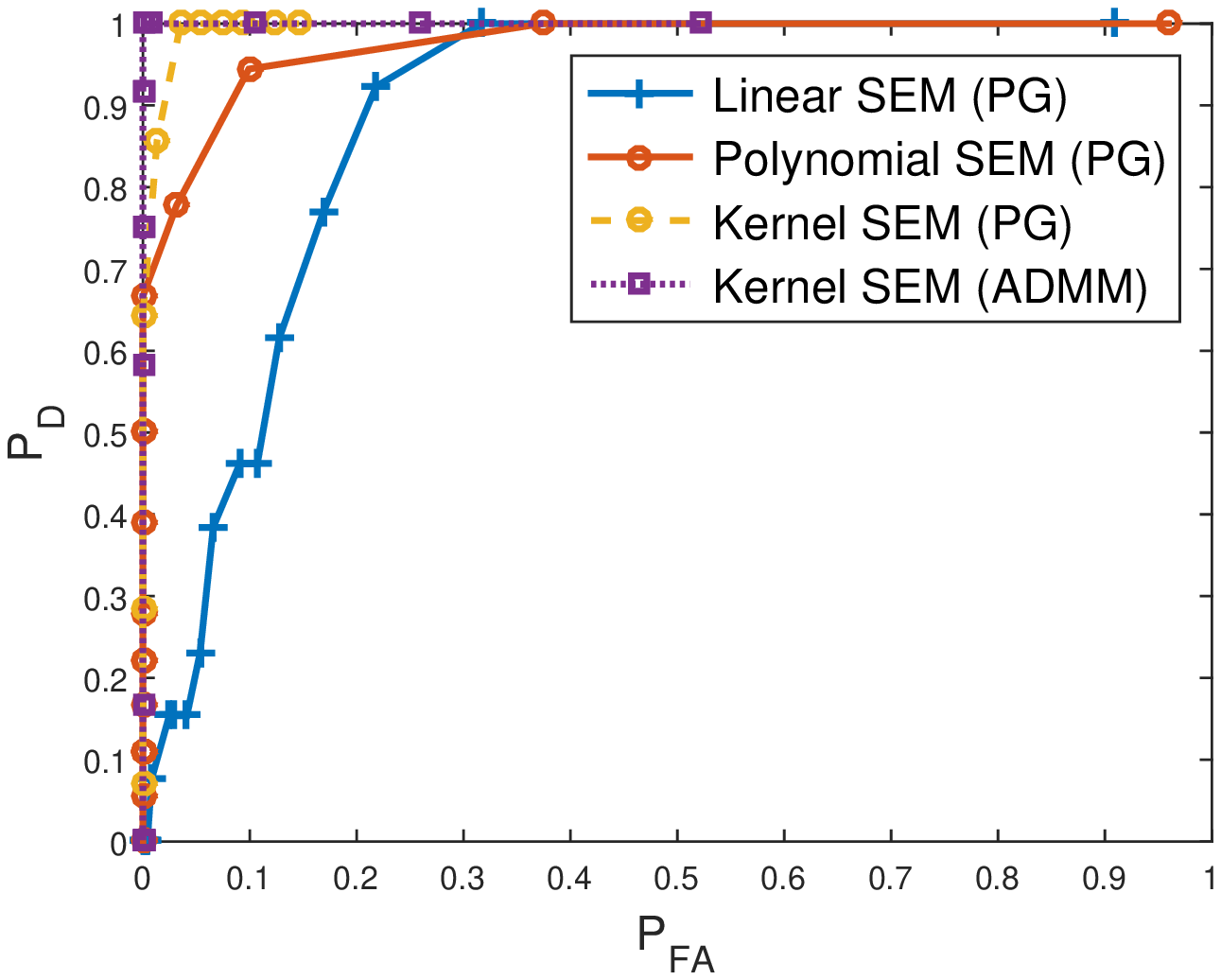}
\centerline{(b)}
\end{minipage}
%
%
%
%
\begin{minipage}[b]{.33\textwidth}
\centering
\includegraphics[width=6.7cm]{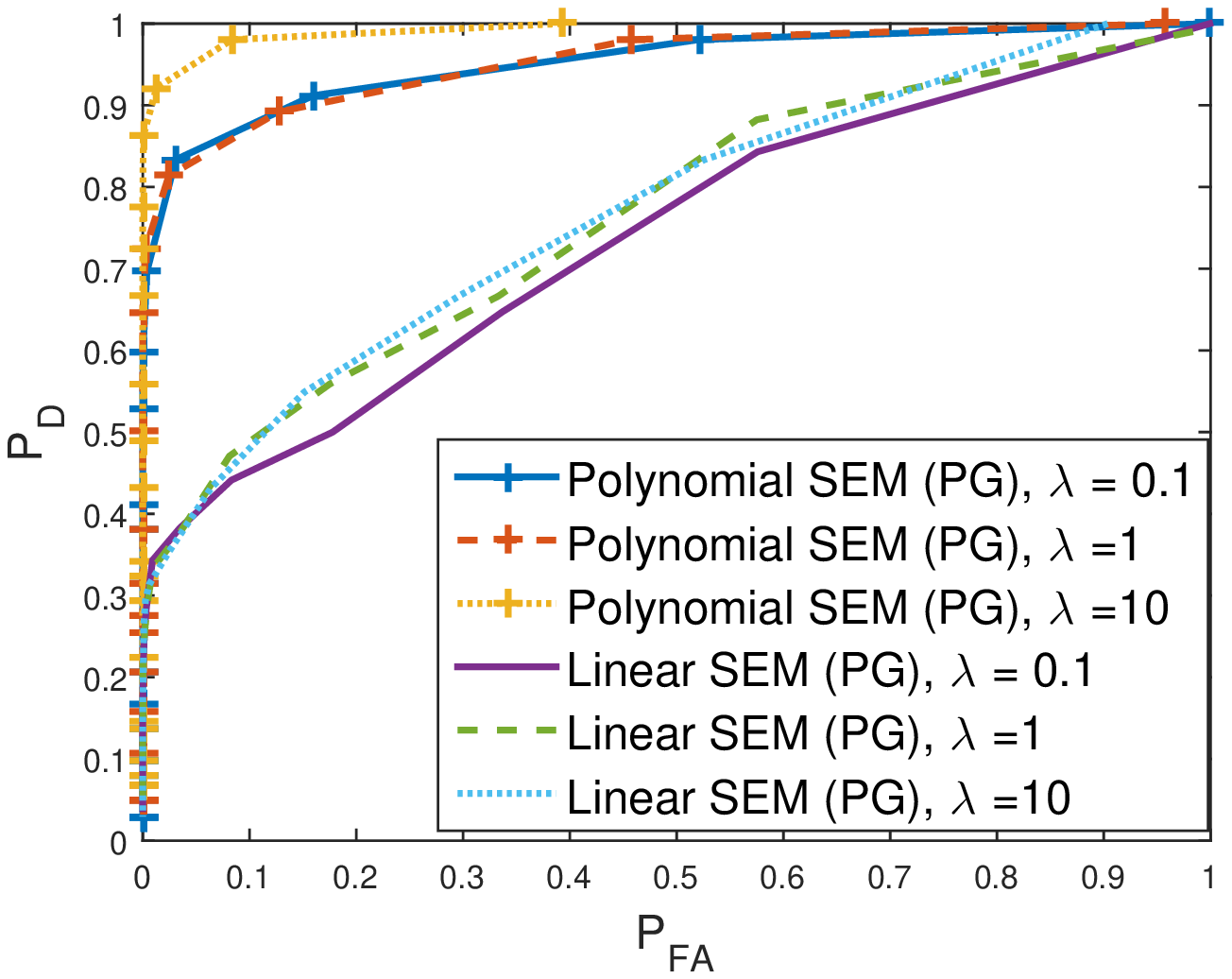}
\centerline{(c)}
\end{minipage}
\vspace{3mm}
 \caption{ROC curves generated under different modeling assumptions: a) nonlinear SEM based on a Gaussian kernel with $\sigma^2=1$; b) nonlinear SEM based on polynomial kernel of order $P=2$; and c) polynomial SEM of order $2$.} 
 \label{fig:roc}
\end{figure*}
\vspace{3mm}
%
%
%


%
%

%
In order to assess edge detection performance, receiver operating characteristic (ROC) curves under different modeling assumptions are plotted in Figure~\ref{fig:roc}. With $P_D$ denoting the probability of detection, and $P_{FA}$ the probability of false alarms, each point on the ROC corresponds to a pair $(P_{FA}, P_D)$ for a prescribed threshold. Figure~\ref{fig:roc} (a) results from tests run on data generated by Gaussian kernels, while Figure~\ref{fig:roc} (b) corresponds to polynomial kernels of order $P=2$. Using the area under the curve (AUC) as the edge detection performance criterion, Figures~\ref{fig:roc} (a) and (b) clearly emphasize the benefits of accounting for nonlinearities. In both plots, kernel-based approaches result in the highest AUC values than approaches that resort to either polynomial or linear SEMs. Interestingly, adopting a polynomial SEM for data generated by Gaussian kernels leads to a lower AUC than the plain linear SEM in this case.

Figure~\ref{fig:roc} (c) plots ROC curves based on linear and polynomial SEMs, with simulated data actually generated using a polynomial SEM of order $2$. The curves are parameterized by the sparsity-control parameter $\lambda$, with more accurate topology identification resulting from $\lambda=10$. As expected, the linear SEM underperforms the polynomial SEM, due to the inherent model mismatch. 

\subsection{Real gene expression data}
Linear SEMs have recently been adopted for identification of gene regulatory topologies, with nodes representing individual genes, while directed edges encode causal regulatory relationships between gene pairs. The goal of this experiment was to test the novel nonlinear SEMs of the present paper, and assess whether it is possible to glean new insights about gene regulatory behavior. 

The experiment was based on real gene expression data resulting from RNA sequencing of cell samples from $69$ unrelated Nigerian individuals, under the International HapMap project~\cite{frazer2007second}. From the $929$ identified genes, expression levels and the genotypes of the \emph{expression quantitative trait loci (eQTLs)} of $39$ immune-related genes were selected and normalized; see~\cite{cai2010memoryless} and ~\cite{pickrell2010understanding} for detailed descriptions. Genotypes of eQTLs were adopted as exogenous inputs or perturbations $\bbX$, since the more typical gene-knockout experiments are generally impractical for human subjects. On the other hand, gene expression levels were treated as the endogenous variables $\bbY$. 

Inference of the underlying gene regulatory network topology was done by adopting both linear and the nonlinear SEM approaches developed in the present paper. For each algorithm, $\lambda$ was carefully selected by $5$-fold cross-validation. Figure~\ref{fig:gene} depicts network visualizations of the identified topologies, with the nodes annotated with their corresponding gene IDs. Figure~\ref{fig:gene} (a) depicts the resulting network based on a linear SEM, while Figures~\ref{fig:gene}(b)-(c) result from nonlinear SEMs based on polynomial kernels of orders $2$ and $3$, as well as a Gaussian kernel with $\sigma^2 = 1$. 

In all cases, the identified networks are very sparse, and the visualizations only include nodes that have at least a single incoming or outgoing edge. Interestingly, the novel nonlinear approaches unveil all edges identified by the linear SEMs, as well as a number of new edges. These newly discovered gene regulatory interactions could potentially be the subject of studies by geneticists, to investigate whether they lead to a better understanding of causal influences among immune-related genes across humans. Clearly, acknowledging the possibility that interactions among genes may be driven by nonlinear dynamics, our novel nonlinear modeling framework subsumes linear approaches, and facilitates discovery of causal patterns that may not be captured through linear SEMs.

\section{Concluding Summary}
\label{sec:conclusion}
This paper put forth a novel nonlinear structural equation modeling framework for inference of sparse directed network topologies, over which observable processes propagate. Postulating a general additive nonlinear model to capture dependencies between endogenous variables, a sparsity-promoting LS estimator was put forth to recover the unknown network topology. Since all dependencies on the unknown functions in the estimator are expressible as inner products, kernels were adopted as an encompassing nonlinear modeling framework. Efficient algorithms based on ADMM and PG iterations were developed to solve the ensuing optimization problem. It was also demonstrated that polynomial SEMs are naturally subsumed by the advocated framework, as a special case in which the nonlinear functions belong to the class of polynomials. Several experiments conducted on simulated data demonstrated the effectiveness of the developed algorithms in inference of sparse directed networks. Numerical tests on real gene expression data demonstrated that the advocated modeling approach is capable of recovering new causal links, that were not detected by conventional linear SEMs. 

This work opens up a number of interesting directions for future research, including: a) broadening the scope of the novel approach to dynamic network topologies; b) exploring more efficient inference algorithms e.g., distributed implementations that are well-motivated in large-scale networks, or online operation when measurements are acquired sequentially; and c) deriving identifiability results for the novel nonlinear model.
\begin{figure*}[t]
\begin{minipage}[b]{.49\textwidth}
\centering
\includegraphics[width=8.4cm]{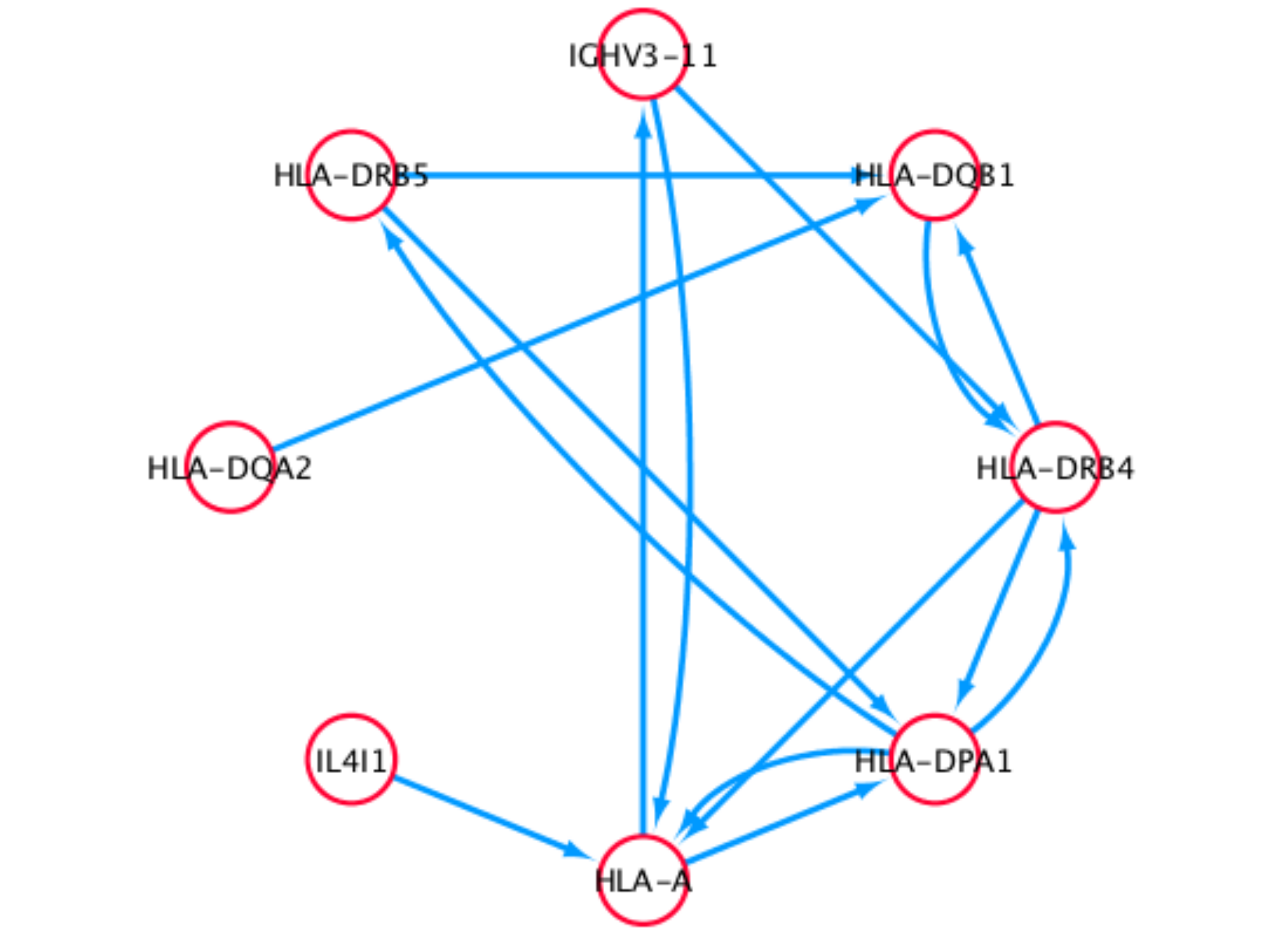}
\centerline{(a)}
\end{minipage}
\begin{minipage}[b]{.49\textwidth}
\centering
\includegraphics[width=8.4cm]{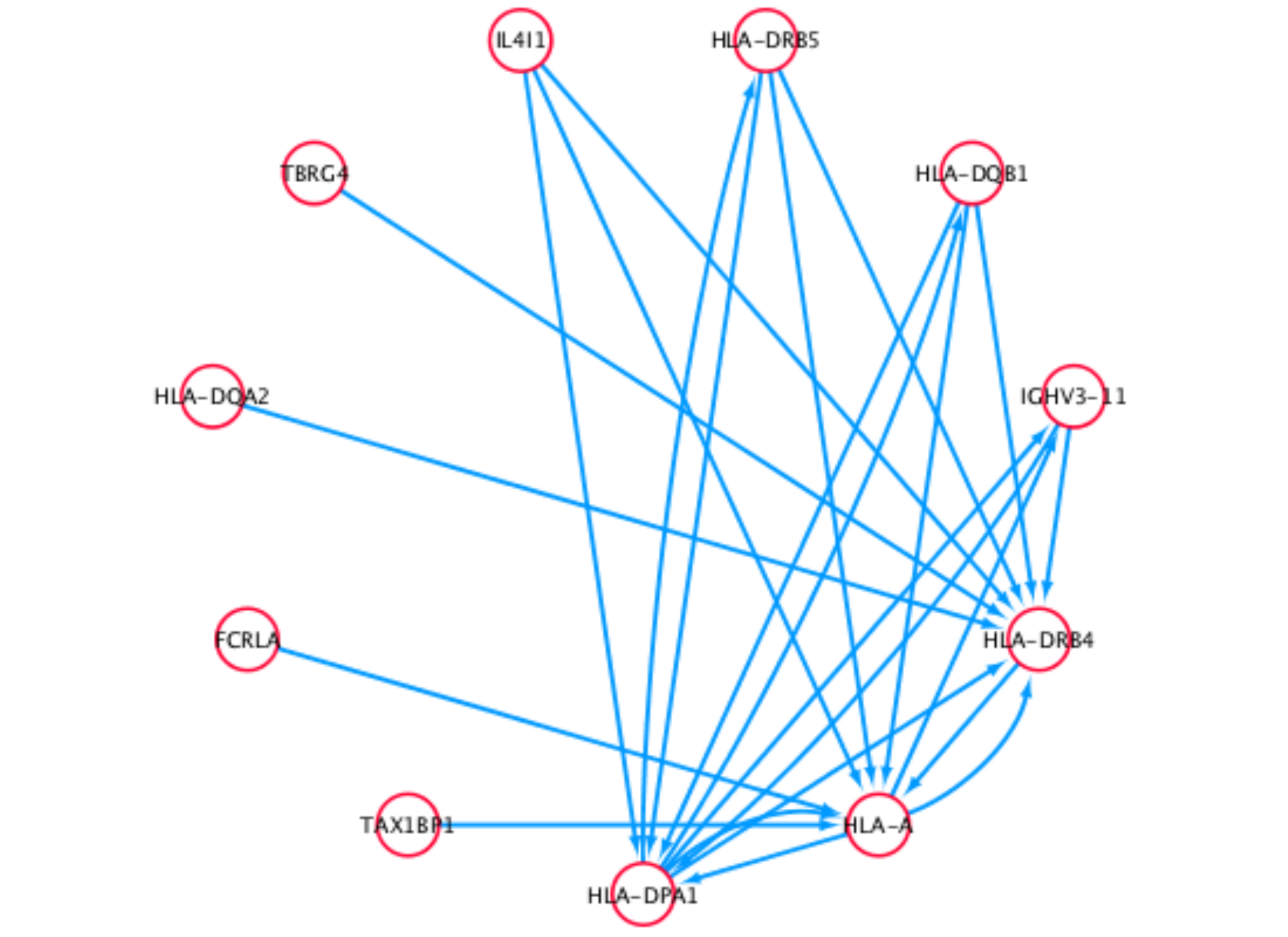}
\centerline{(b)}
\end{minipage}
\begin{minipage}[b]{.49\textwidth}
\centering
\includegraphics[width=8.4cm]{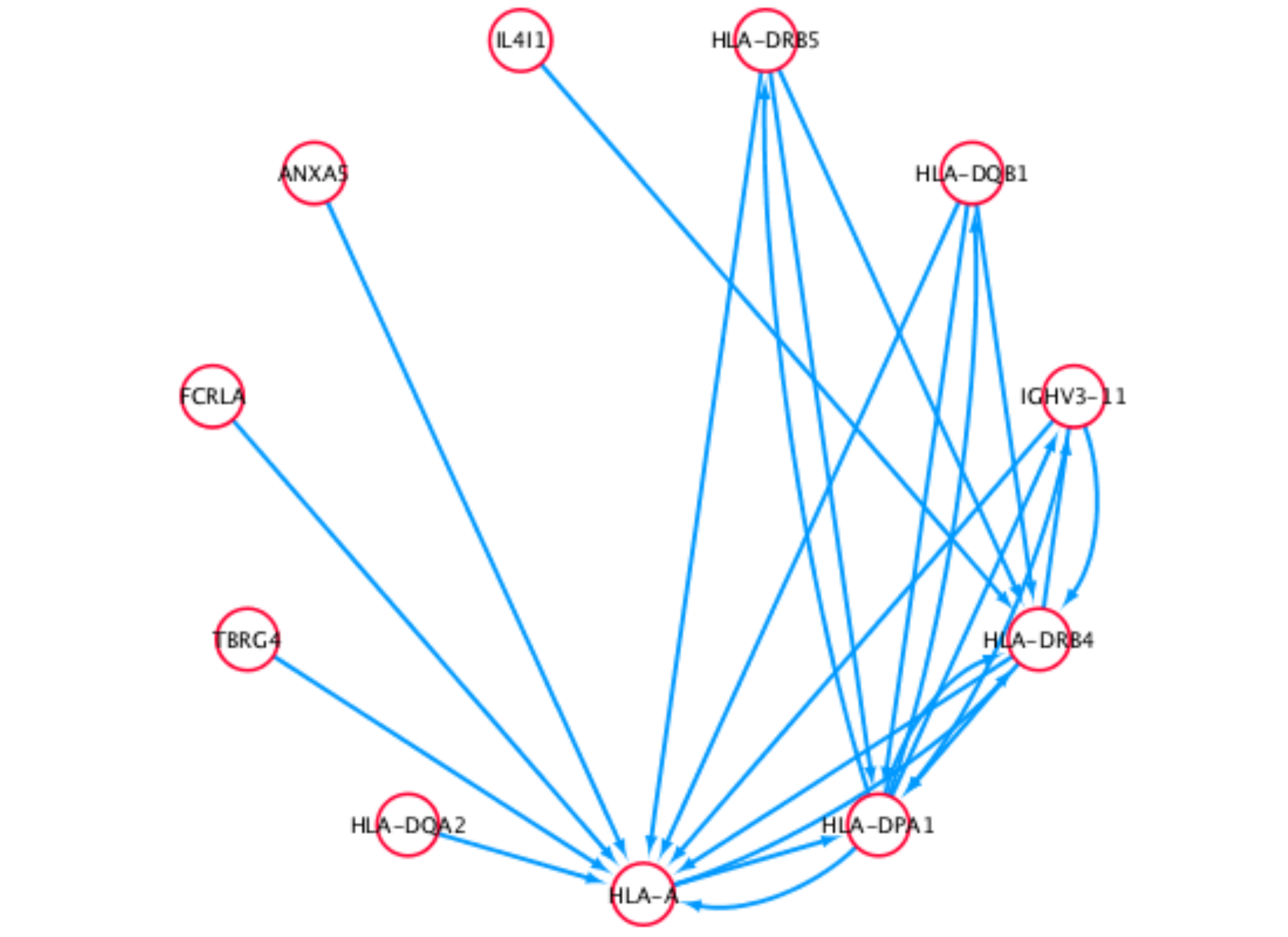}
\centerline{(c)}
\end{minipage}
\begin{minipage}[b]{.49\textwidth}
\centering
\includegraphics[width=8.4cm]{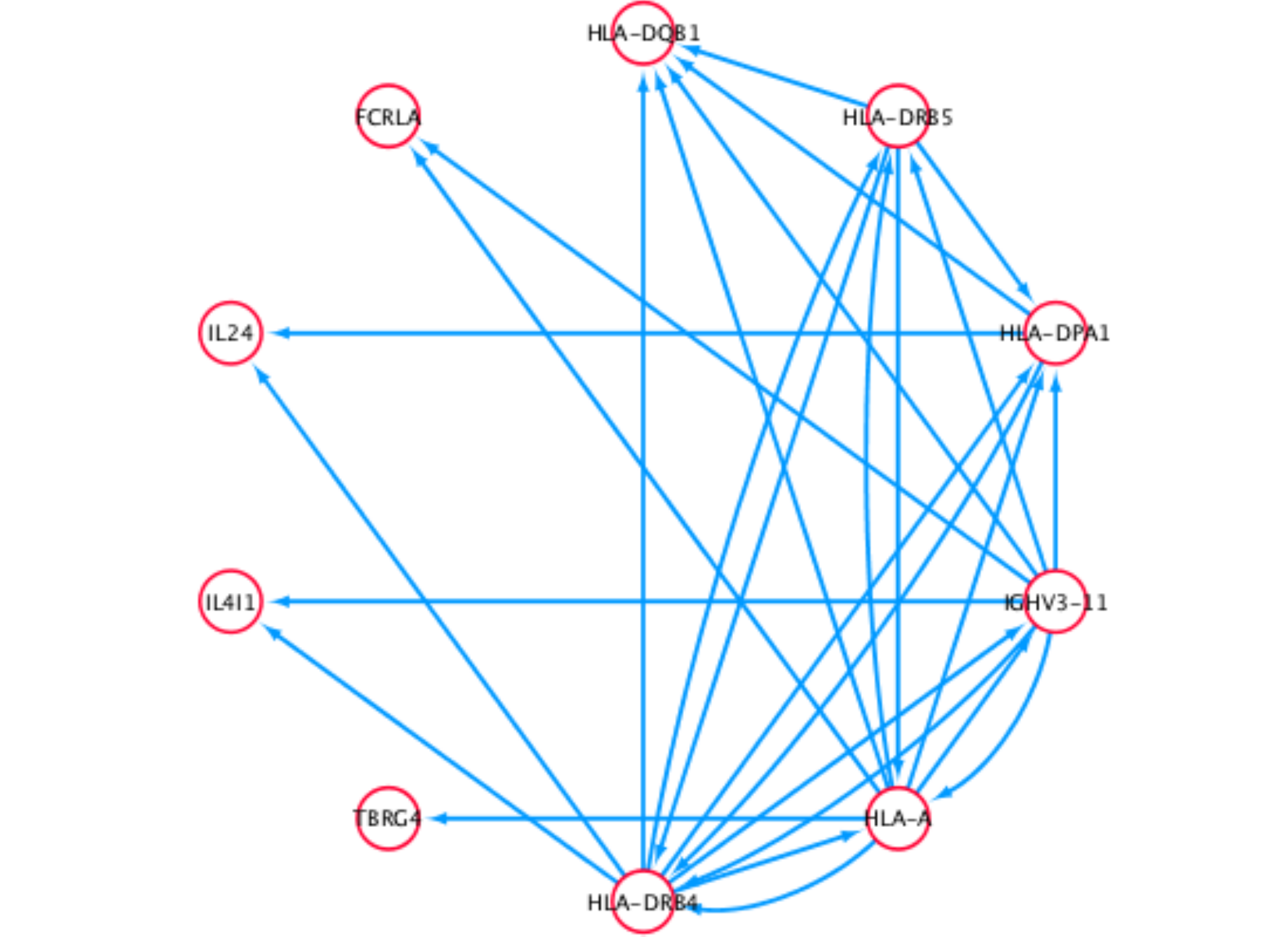}
\centerline{(d)}
\end{minipage}
 \caption{Network visualizations of inferred gene regulatory networks of $39$ immune-related human genes, based on gene expression data of $M=69$ individuals using: (a) a linear SEM; (b) a kernel-based SEM using polynomial kernels of order $2$; (c) a kernel-based SEM using polynomial kernels of order $2$; and (d) a kernel-based SEM using Gaussian kernels with $\sigma^2=1$. The nonlinear inference approaches are capable of unveiling a number of new links that were not discovered by linear SEMs.} 
 \label{fig:gene}
\end{figure*}
\appendices
\section{Proof of Proposition \ref{proposition1}}
\label{appendix:A}

First notice that~\eqref{eq:obj:matrix} is separable across columns of $\bbW$. The separable cost can be modified to incorporate the constraints, with each subproblem consequently becoming expressible as
\begin{multline}
\label{eq:app:vec}
	\{\{\hat{\bbw}_{ij}\}_{i\neq j}, \hat{b}_{jj}\}= \\ 
	\underset{{\{\bbw_{ij}\}, b_{jj}} }{\min} \;\;
	  (1/2) \big \| \bby_j-  \sum_{i\neq j}\bbPhi_i\bbw_{ij} 
	  -b_{jj}\bbx_j \big \|_2^2 
	+\lambda \sum_{i\neq j}\|\bbw_{ij}\|_2.
\end{multline}
Equivalently, by isolating $\bbw_{ij}$ for a specific $i$, \eqref{eq:app:vec} can now be written as
\vspace{1mm}
\begin{multline}
\label{eq:app:1}
	\min_{\{{\bbw}_{kj}\}_{k\neq i,j}, b_{jj}} \bigg[  \min_{\bbw_{ij}} \;\; (1/2) \big \| \bby_j-\sum_{k\neq i, j}\bbPhi_k\bbw_{kj} \\ 
	-b_{jj}\bbx_j-\bbPhi_i\bbw_{ij} \big \|_2^2
	+\lambda \|\bbw_{ij}\|_2 \bigg]+\lambda \sum_{k\neq i, j}\|\bbw_{kj}\|_2
\end{multline}
where the inner optimization problem is solved first, followed by the outer one. 
Defining terms that are constant w.r.t. $\bbw_{ij}$ in~\eqref{eq:app:1} as the vector
\vspace{1mm}
\begin{equation}
\label{eq:app:1b}
\bbr_{ij} := \bby_j - \sum_{k\neq i, j} \bbPhi_k\bbw_{kj} - b_{jj}\bbx_j
\end{equation}
and focusing only on the inner minimization over $\bbw_{ij}$, one obtains the following regularized LS problem
\vspace{2mm}
\begin{equation}
\label{eq:app:2}
	\min_{\bbw_{ij}}  \;\; (1/2) \|\bbr_{ij}-\bbPhi_i\bbw_{ij}\|_2^2+\lambda \|\bbw_{ij}\|_2.
\end{equation}
Recalling that  $\bbPhi_{i} := [\bbphi(y_{i1}),\ldots,\bbphi(y_{iM})]^\top 
$, the set of vectors 
\[
\mathcal{F} := \left\lbrace \bbphi(y_{i1}), \ldots, \bbphi(y_{iM}),\bbd_1,\ldots,\bbd_{P-M'} \right\rbrace
\]
constitutes a complete basis of $\mathbb{R}^P$, where $M':= \text{dim}(\text{span}\left(\bbphi(y_{i1}), \ldots, \bbphi(y_{iM})\right))$, and $\{\bbd_k\in \mathbb{R}^P\}_{k=1}^{P-M'}$ denotes nonzero vectors orthogonal to $\{\bbphi(y_{im})\}$, i.e.,
\vspace{1mm}
\begin{equation}
\label{eq:app:3}
\bbd_k^\top \bbphi(y_{im})=0 \; \forall k,m.
\end{equation}
This means that any vector in $\mathbb{R}^P$ can be represented as a linear combination of elements
from $\mathcal{F}$, and the optimal solution to~\eqref{eq:app:2} admits the following expansion
\vspace{1mm}
\begin{align}
\label{eq:app:4}
	\hat{\bbw}_{ij}=\sum_{m=1}^M \alpha_{ijm} \bbphi(y_{im})+\sum_{k=1}^{P-M'} \beta_{ijk}\bbd_k
\end{align}
without loss of generality. Note that $\{ \alpha_{ijm} \}$ and $\{ \beta_{ijk} \}$ in~\eqref{eq:app:4} denote 
basis coefficients. Adopting~\eqref{eq:app:4}, the LS fitting term in \eqref{eq:app:2} can be written as
\begin{align}
\label{eq:app:5}
	&\|\bbr_{ij}-\bbPhi_i\hat{\bbw}_{ij}\|_2^2\nonumber\\
	=&\left\|\bbr_{ij}-\bbPhi_i\left(\sum_{m=1}^M \alpha_{ijm} \bbphi(y_{im})+\sum_{k=1}^{P-M'} \beta_{ijk}\bbd_k\right)\right\|_2^2\nonumber\\
	=&\left\|\bbr_{ij}-\sum_{m=1}^M \alpha_{ijm}\bbPhi_i \bbphi(y_{im})\right\|_2^2
\end{align}
where the last equality is a direct consequence of~\eqref{eq:app:3}. Clearly, from~\eqref{eq:app:5}, the component of $\bbw_{ij}$ lying within the span of $\{\bbd_k\}$ has no influence on the LS fitting term in~\eqref{eq:app:4}. On the other hand, the second term in~\eqref{eq:app:2} can be expanded as [cf.~\ref{eq:app:4}]
\begin{align}
\label{eq:app:6}
	\|\hat{\bbw}_{ij}\|_2&=\left\|\sum_{m=1}^M \alpha_{ijm} \bbphi(y_{im})\right\|_2+\left\|\sum_{k=1}^{P-M'} \beta_{ijk}\bbd_k\right\|_2\nonumber\\
	&\geq \left\|\sum_{m=1}^M \alpha_{ijm} \bbphi(y_{im})\right\|_2
\end{align}
where the first equality holds due to~\eqref{eq:app:3}, while the second inequality follows from the non-negativity of norms. Also note that equality in~\eqref{eq:app:6} is attained when $\beta_{ijk}=0,~\forall k$. 

Combining~\eqref{eq:app:5} and~\eqref{eq:app:6}, one deduces that $\hat{w}_{ij}$ is the optimal solution to~\eqref{eq:app:2} if $\beta_{ijk}=0,~\forall k$ in~\eqref{eq:app:4}. This is based on the argument that any coefficient $\beta_{ijk} \neq 0$ will not change the value of the LS fit, yet it will certainly increase the penalty term, leading to a higher overall cost. It can now be concluded that any optimal solution to~\eqref{eq:app:2}, regardless of the value of $\bbr_{ij}$, admits the following expansion 
\vspace{1mm}
\begin{align}
\label{eq:app:7}
	\hat{\bbw}_{ij}=\sum_{m=1}^M \alpha_{ijm} \bbphi(y_{im})=\bbPhi_i^\top \bbalpha_{ij}.
\end{align}
It is also worth noting that~\eqref{eq:app:7} holds for any $i$ regardless of the outer optimization in~\eqref{eq:app:vec}. This concludes the proof of Proposition~\ref{proposition1}.

\bibliography{net,myabrv}
\bibliographystyle{IEEEtranS}
\end{document}